\documentclass[letterpaper, 10 pt, conference]{ieeeconf}  % Comment this line out if you need a4paper

\IEEEoverridecommandlockouts                              % This command is only needed if
\usepackage{times}
\usepackage{epsfig}
\usepackage{graphicx}
\usepackage{mathptmx}
\usepackage{amsmath}
\usepackage{amssymb}
\usepackage{bbm}
\usepackage{cite}
\usepackage{algorithm}
\usepackage[utf8]{inputenc}
\usepackage[english]{babel}
\usepackage{xcolor}
\usepackage{pagecolor}
\usepackage[noend]{algpseudocode}
\usepackage{booktabs}
\usepackage{multirow}
\usepackage{caption}
\usepackage{subcaption}

\usepackage[pagebackref=true,breaklinks=true,letterpaper=true,colorlinks,bookmarks=false]{hyperref}

\newtheorem{lemma}{Lemma}
\newtheorem{theorem}{Theorem}

\DeclareMathOperator{\avg}{avg} % no space, limits on side in displays
\DeclareMathOperator{\argmax}{argmax} % no space, limits on side in displays

\usepackage[pagebackref=true,breaklinks=true,letterpaper=true,colorlinks,bookmarks=false]{hyperref}

\title{\LARGE \bf
Learning to Optimally Segment Point Clouds
}

\author{Peiyun Hu$^{1}$, David Held$^{1}$\textsuperscript{*}, Deva Ramanan$^{1,2}$\textsuperscript{*}\\
$^{1}$Robotics Institute, Carnegie Mellon University\\
$^{2}$Argo AI\\
{\tt \small \{peiyunh@cs, dheld@andrew, deva@cs\}.cmu.edu}
\thanks{\textsuperscript{*} indicates two authors have equal contribution.}% <-this % stops a space
}

\begin{document}

\maketitle
\thispagestyle{plain}
\pagestyle{plain}

%%%%%%%%%%%%%%%%%%%%%%%%%%%%%%%%%%%%%%%%%%%%%%%%%%%%%%%%%%%%%%%%%%%%%%%%%%%%%%%%
\begin{abstract}
%
%This electronic document is a live template. The various components of your paper [title, text, heads, etc.] are already defined on the style sheet, as illustrated by the portions given in this document.
%
%Perception for autonomous vehicles presents a collection of unique challenges: (1) finding the right presentation for 3D signals; (2) adapting to an open-world setting; (3) exploiting geometric perceptual priors. In this paper,
We focus on the problem of class-agnostic instance segmentation of LiDAR point clouds. We propose an approach that combines graph-theoretic search with data-driven learning: it searches over a set of candidate segmentations and returns one where individual segments score well according to a data-driven point-based model of ``objectness''. We prove that if we score a segmentation by the worst objectness among its individual segments, there is an efficient algorithm that finds the optimal worst-case segmentation among an exponentially large number of candidate segmentations. We also present an efficient algorithm for the average-case. For evaluation, we repurpose KITTI 3D detection as a segmentation benchmark and empirically demonstrate that our algorithms significantly outperform past bottom-up segmentation approaches and top-down object-based algorithms on segmenting point clouds.
\end{abstract}

\section{Introduction}
% vision for robotics
%% what is the right representation for sparse points?
%% Open-world vs closed-world understanding
%% Map-based priors
Perception for autonomous robots presents a collection of compelling challenges for computer vision. We focus on the application of autonomous vehicles. This domain has three notable properties that tend not to surface in traditional vision applications: (1) 3D sensing in the form of LiDAR technology, which exhibits different properties than traditional 3D vision captured through stereo or structured light. Despite significant work in this area, the right representation for such sparse 3D signals still remains an open question. (2) Contemporary approaches to object detection and scene understanding tend to be closed-world, where the task is predicting 1-of-N possible labels. But autonomous systems require the ability to recognize all possible obstacles and movers - e.g., a piece of road debris must be avoided regardless of what \emph{name} it has. Such understanding is crucial from a safety perspective. Historically, this has been formulated as a perceptual grouping or bottom-up segmentation task, which is typically addressed with different approaches. (3) Finally, practical autonomous robotics makes heavy use of perceptual priors in the forms of geometric maps and assumptions on LiDAR geometry.  Indeed, prior map was a crucial component among finishing entries in the DARPA Urban Grand Challenge~\cite{urmson2008autonomous,montemerlo2008junior}.

\begin{figure}[t]
  \centering
  \includegraphics[clip=true,trim={0 10cm 0 5cm},width=\linewidth]{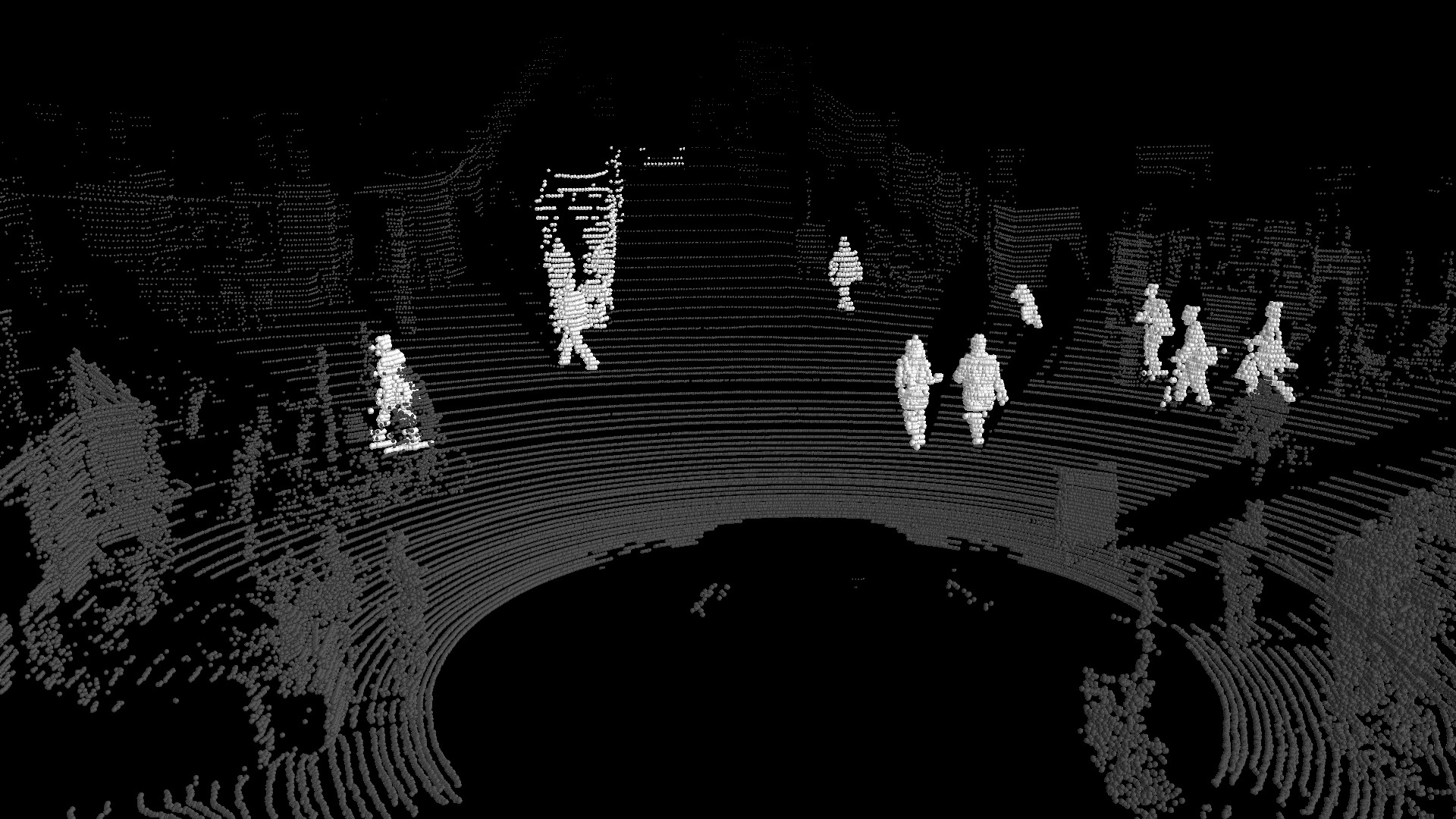}\\[.1em]
  \includegraphics[clip=true,trim={0 10cm 0 5cm},width=\linewidth]{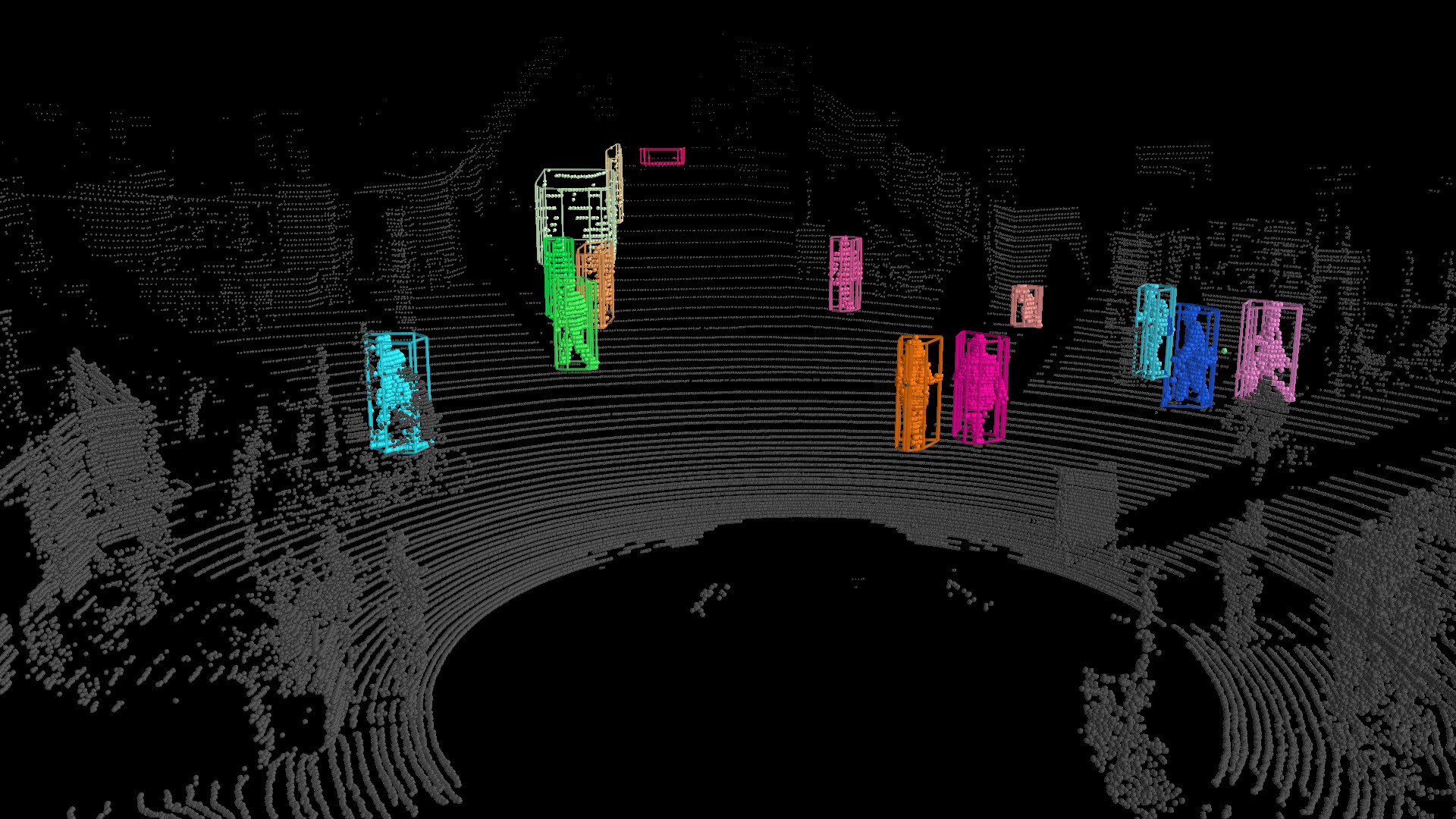}
  \caption{Our proposed algorithm takes a pre-processed LiDAR point cloud with background removed (top) and produces a class-agnostic instance-level segmentation over all foreground points (bottom). For visualization, we use a different color for each segment and plot an extruded polygon to show the spatial extent.}
  \label{fig:splash}
\end{figure}

{\bf Motivation:} In this work, we focus on the problem of class-agnostic instance segmentation of LiDAR point clouds (Figure~\ref{fig:splash}) in an open-world setting. We carefully mix graph-theoretic algorithms with data-driven learning. Data-driven learning has made an undeniable impact on computer vision, but is difficult to make guarantees about performance when processing out-of-sample data from an open world. Geometric graph-based approaches for segmentation tend not to require training and so are less-like to overfit, but also tend to be brittle.

{\bf Approach:} Our approach searches over an exponentially-large space of candidate segmentations and returns one where individual segments score well according to a data-driven point-based model of ``objectness"~\cite{alexe2010object}. We demonstrate that one can repurpose existing closed-world point networks~\cite{qi2017pointnet++} for bottom-up perceptual grouping tasks that generalize to objects rarely seen during training.

{\bf Optimality:} We prove that our approach produces \emph{optimal} segmentations according to a specific definition. First, we restrict the search into a subset of segmentations that are consistent with a hierarchical grouping of a point cloud sweep. Such hierarchical groups can be readily produced with agglomerative clustering~\cite{zhou2009streaming}, HDBSCAN~\cite{mcinnes2017hdbscan}, or hierarchical graph-based algorithms~\cite{strom2010graph}.

Naive methods for producing a segmentation might apply a global threshold over the whole hierarchy. It turns out that one can produce an exponentially-large set of segmentations by applying different thresholds at different branches. We introduce efficient algorithms that search over this space of tree-consistent segmentations (Figure~\ref{fig:tree}) and return the one that maximizes a global segmentation score that is computed by aggregating local objectness scores of individual segments.

{\bf Evaluation:} We demonstrate empirical results on KITTI, a benchmark originally designed for closed-world object detection. Following past work, we repurpose it for open-world 3D segmentation~\cite{held2016probabilistic}. We compare to existing bottom-up approaches~\cite{rusu2010semantic} and state-of-the-art LiDAR-based object detectors after converting their output 3D bounding boxes to a point cloud segmentation. We demonstrate that our approaches outperform both baselines on less common classes.

\section{Related work}

Robust 3D object detection is crucial for downstream applications such as semantic understanding~\cite{hackel2016fast} and tracking~\cite{DBLP:journals/corr/abs-1907-03961}. Comparing to monocular 3D detection~\cite{ku2019monocular}, we focus on LiDAR-based solutions in this paper.

{\bf LiDAR segmentation:} Classic LiDAR segmentation algorithms use bottom-up grouping such as flood-filling~\cite{douillard2011segmentation}, connected components~\cite{klasing2008clustering}, or density-based clustering~\cite{mcinnes2017hdbscan}. Bottom-up strategies can also be applied on LiDAR sequences, allowing for motion as an additional cue~\cite{held2014combining,teichman2011towards,azim2012detection}. Oftentimes such approaches are tuned for particular object categories such as cars. Our work differs in its use of static, single-frame cues that are not object-specific.

{\bf LiDAR object detection:} There is an ever-increasing literature on data-driven object detection with LiDAR point clouds. Early approaches include fusion-based models that combine LiDAR and imagery~\cite{ku2018joint}, tracking-based detectors~\cite{luo2018fast} and voxel-based classifiers~\cite{zhou2018voxelnet,lang2018pointpillars,yan2018second}. We have seen approaches built upon raw point clouds such as PointRCNN~\cite{shi2000normalized}. Our approach is most related to Frustum PointNet~\cite{qi2018frustum} in the way we use pooled point cloud representation~\cite{qi2017pointnet++}. Our work differs in that we do not make use of camera input, and most notably, focus on \emph{all} possible objects in an open world. Specifically, we compare to \cite{ku2018joint,lang2018pointpillars,yan2018second,shi2019pointrcnn} as a representative sample of the literature.

{\bf Perceptual grouping:} Our graph-based approach is inspired by a long line of classic work on graph-theoretic perceptual grouping, dating back to normalized cuts~\cite{shi2000normalized}, graph cuts~\cite{boykov1999fast}, and spanning-tree approaches~\cite{felzenszwalb2004efficient}. Such methods are typically used with hand-designed features, while we make use of data-driven techniques for learning a shape-based segment classifier.

{\bf Image segmentation:} The idea of searching for an optimal image segmentation given a hierarchical image segmentation tree has been explored. \cite{uzunbas2016efficient} formulates neuron segmentation on electron microscopy images as a \emph{maximum a posteriori} (MAP) labeling task on a tree-structured graph. It can be made equivalent to our search under certain conditions. \cite{ovsep2018track} tackles the problem of class-agnostic instance segmentation in image space by exploiting visual appearance and motion. We discuss more in Section~\ref{sec:approach} and~\ref{sec:results}.

\begin{figure}
  \centering
    \includegraphics[height=3cm]{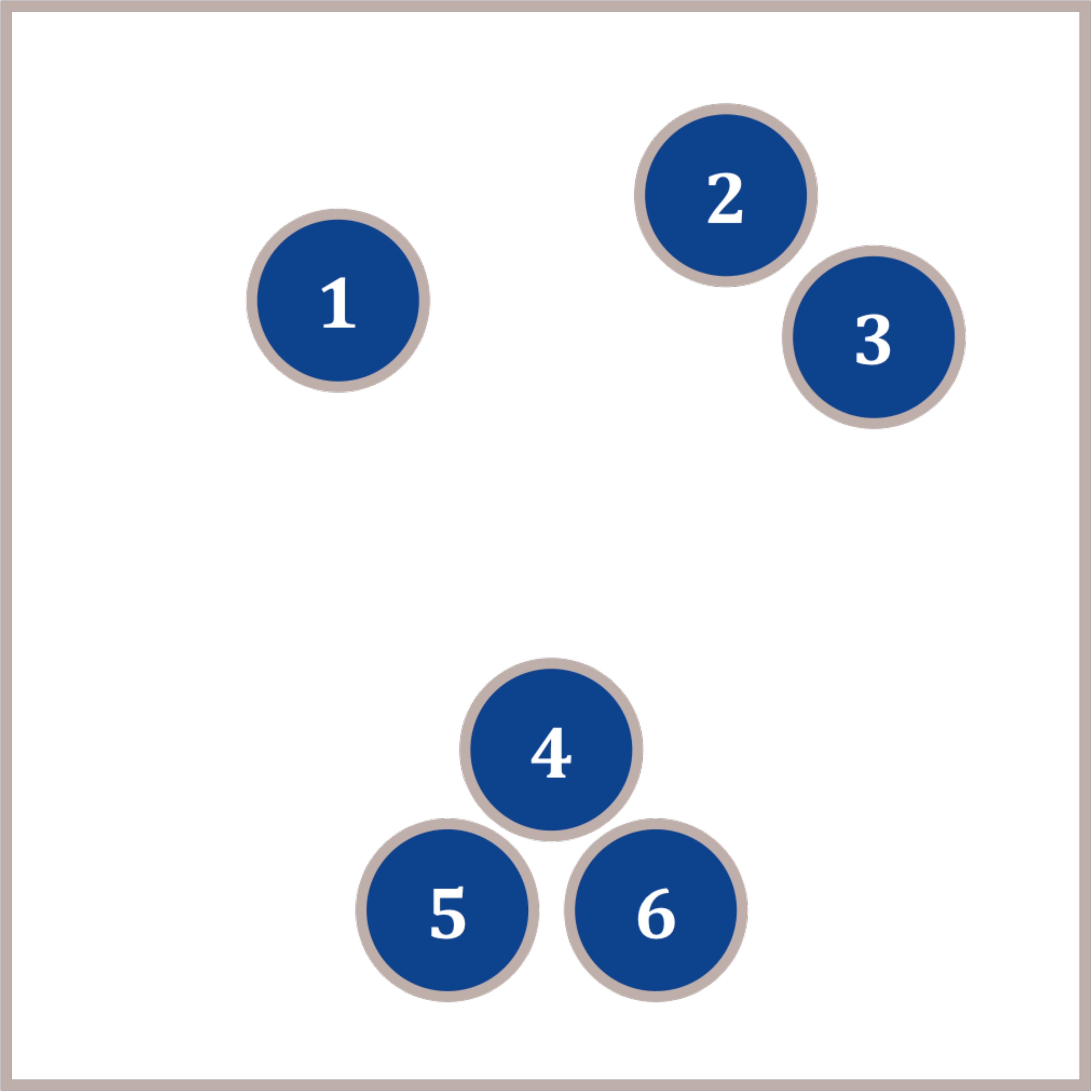}
    \includegraphics[height=3cm]{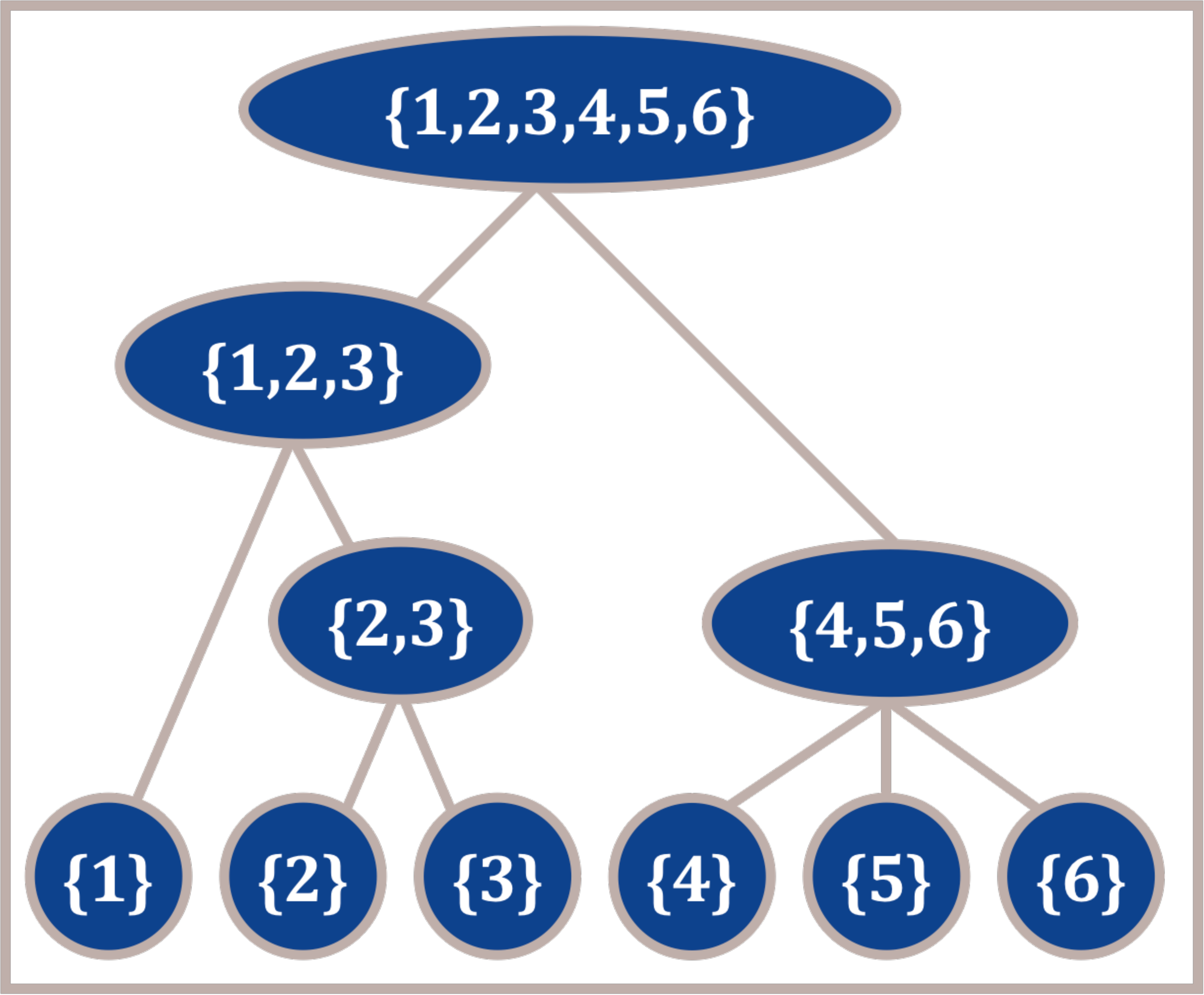}
\caption{On the left, we visualize a set with 6 points. According to Bell number, one will find 203 unique segmentations (partitions). Most of these are arbitrary and do not respect local geometry, e.g. $\{\{1,2,5\},\{3,4,6\}\}$. On the right, we implement geometric constraints with a tree formed by hierarchical grouping. Every vertex cut of this tree is automatically a segmentation that respects local geometry encoded by the tree, e.g. $\{\{1\}, \{2,3\}, \{4,5,6\}\}$.}
  \label{fig:tree}
\end{figure}

\section{Approach}
\label{sec:approach}

For 3D object point segmentation, the input is a 3D point cloud, which contains an \textit{unknown} number of objects. The goal is to produce a point segmentation, in which every segment contains points from one and only one object.

{\bf Segmentation:} A global \emph{segmentation} $P_X$ is a partition of a set of points $X=\{x_i\}_{i=1}^N$ into subsets of points, i.e. $P_X=\{C_i\}_{i=1}^M$, where $M$ denotes the number of segments and $C_i \subset X$. We refer to each $C_i$ as a local \emph{segment}. Importantly, every point exists in one and only one segment, meaning $\cup_{i=1}^M C_i=X$ and $\forall i \neq j, C_i \cap C_j =\emptyset$.

{\bf Tree-consistent segmentations:} Let us use $S_X$ to denote the set of all possible global segmentations on $X$, i.e.\ all possible $P_X$. Without constraints, the size of $S_X$ is exponential in $N$ (i.e.\ the Bell number). In practice, we can reduce the number of candidates by enforcing geometric constraints. In this work, we implement the constraints by grouping all points hierarchically into a tree structure $T_X$. We will discuss how to build such a tree structure based on local geometric cues in Section~\ref{sec:tree}. For now let us assume the tree is given.

Once we specify the tree, we can focus on a strictly smaller set of segmentations that respect local geometry. We denote such set as $S_{X,T}$ and call them \emph{tree-consistent} segmentations. As a reference, the size of $S_{X,T}$ is still exponential in $N$, when $T_X$ is a balanced binary tree~\footnote{One can derive recurrence on the number of segmentations between depth d+1 and d as $K_{d+1}=K_d^2+1$ with $K_1=2$. Since $K_{d}>2^{2(d-1)}$, $K_d/N_d > 2^{d-2}$, where $N_d=2^d$ represents the number of leaves, it suggests the number of segmentations \textit{at least} outgrow the number of leaves exponentially.}. We further illustrate the relationship between $S_X$ and $S_{X,T}$ with an example in Figure~\ref{fig:tree}. Any tree-consistent segmentation from $S_{X,T}$ corresponds to a vertex cut set of the tree $T$, i.e.\ a set of tree nodes, which satisfy the following constraints: (1) for each node in the vertex cut, its ancestor and itself cannot both be in the cut and (2) each leaf node must have itself or its ancestor in the cut. Such relationship allows us to design efficient tree searching algorithms, as we will see later.

{\bf Segment score:} Before we discuss how to score a \emph{global} segmentation, we first introduce how to score a \emph{local} segment. Given a local segment $C \subset X$, we define a function $f(C; \theta): C \mapsto [0,1]$ that predicts a given segment's ``objectness'', where $\theta$ represents the parameters. One can implement such a function with a PointNet++, where $\theta$ would represent weights of the PointNet++. We will discuss how to learn this function in Section~\ref{sec:learn}. For now let us assume it is given.

{\bf Segmentation score:} We now introduce how to score a \emph{global} segmentation. Given a global segmentation $P_X=\{C_i\}_{i=1}^M$, we define its score $F(P_X; \theta): P_X\mapsto [0,1]$ by aggregating over local objectness of its individual segments. Specifically, we introduce \emph{worst-case} segmentation and \emph{average-case} segmentation. Note that our objective can be made equivalent to \cite{uzunbas2016efficient} if we score a segmentation as the \emph{sum} of its local segment scores. As we see in Section~\ref{sec:results}, this objective produces much larger oversegmentation error.

\subsection{Worst-case segmentation}
\emph{Worst-case} segmentation scores a global segmentation as the \emph{worst} objectness among its local segments:
\begin{equation}
  \label{eq:worst}
  F_{\min}(P_X; \theta) = \min_i{f(C_i; \theta)}, i\in {1 \ldots M}
\end{equation}
where $P_X \in S_{X,T}$, $P_X=\{C_i\}_{i=1}^M$, and $C_i \subset X$. We define $P_{X, min}^*$ as the \textit{optimal worst-case segmentation} if
\begin{equation}
  \label{eq:optimal-worst}
  P_{X,\min}^* = \argmax_{P_X\in S_{X,T}} {F_{\min} (P_X; \theta)}
\end{equation}

It turns out the problem of finding optimal worst-case segmentation has optimal substructure (Theorem~\ref{theorem:optimal-worst}), allowing us to find the global optimum efficiently with dynamic programming (Algorithm~\ref{alg:optimal-worst}).

We briefly describe how the algorithm works. Given a set of points $X$ and a tree $T_X$, \textproc{OptMinSeg($X$, $T_X$)} (Algorithm~\ref{alg:optimal-worst}) produces an optimal worst-case segmentation $P^*_{X,\min}$ with score $F^*_{\min}(P^*_{X,\min{}};\theta)$. For simplicity, we refer to a node in the tree by the set of points it is associated with. The algorithm starts from the root node $X$ and chooses between a coarse segmentation ($\{X\}$) and a fine one. The fine segmentation will be the union of all $X$'s children's optimal worst-case segmentation, which can be computed recursively. The algorithm would first traverse down to the leaf nodes, representing the finest segmentation. Then it will make its way up, during which it finalizes optimal segmentations for each intermediate node by making local coarse vs.\ fine decisions. Eventually, it returns to the root node and produces an optimal worst-case global segmentation.

\begin{lemma}
  \label{lemma:min}
  Given pairs of non-empty sets that contain real numbers $(X_1, Y_1),\ldots, (X_n, Y_n)$,
  \begin{equation}
    \forall i, \min_{x \in X_i}{x} \leq \min_{y \in Y_i}{y} \Rightarrow \min_{x \in \cup_i X_i}{x} \leq \min_{y \in \cup_i Y_i}{y}
  \end{equation}
\end{lemma}

\begin{theorem}
  \label{theorem:optimal-worst}
  Given $C$ and $T_C$, Algorithm~\ref{alg:optimal-worst} finds the optimal segmentation $P_{C,min}^* = \argmax_{P_C \in S_{C,T}} F_{min} (P_C; \theta)$.
\end{theorem}

\begin{proof}
  %TODO: fix counter of proof
  \label{proof:optimal-worst}
  %https://courses.engr.illinois.edu/cs173/fa2010/Lectures/trees.pdf
  Proof by structural induction.

  \noindent \textbf{Base:} When $N_C=\emptyset$, meaning $C$ corresponds to a leaf node in $T_C$, the algorithm returns $\{C\}$, which is the only segmentation in $S_{C,T}$ and obviously is optimal.

  \noindent \textbf{Induction:}
When $N_C\neq \emptyset$, we need to show that the algorithm will produce the optimal segmentation, i.e. $P^*_{C}$ and $F^*_{C}$, if it has access to the optimal segmentation for each of $C$'s child $C_i$, i.e. $P^*_{C_i}$ and $F^*_{C_i}$ (optimal substructure).

  Let $P_C$ be the segmentation that the algorithm produces for $C$ and let $F_C$ be its score. If $P_C$ were not optimal, there must exist a different segmentation $P''_C$ with score $F'_C$, s.t. $P'_C\neq P_C$ and $F'_C > F_C$. Moreover, $P'_C$ is either a trivial segmentation, i.e. $P'_C = \{C\}$ or the union of segmentations over each of $C$'s children nodes, i.e. $P'_C = \cup_i \{P'_{C_i}\}$.

  First, $P'_C$ is not a trivial segmentation. If we assume $P'_C=\{C\}$, we will have $F'_C=f(C;\theta)$. Since $P_C\neq P'_C$, the algorithm chooses $P_C$ over $\{C\}$, therefore, $F_C>f(C;\theta)$. This clearly contradicts with $F'_C > F_C$.

  Thus, $P'_C$ has to be the union of segmentations over each of $C$'s children node. According to the inductive hypothesis, the algorithm has the optimal segmentation over each of $C$'s children node, meaning $\forall i, F'_{C_i} \leq F^*_{C_i}$ or concretely
  \begin{equation}
      \forall i, \min_{z \in P'_{C_i}} f(z;\theta) \leq \min_{z \in P^*_{C_i}} f(z;\theta)
  \end{equation}
  Here, $z$ represents an arbitrary local segment from a segmentation over $C_i$. By applying Lemma~\ref{lemma:min}, we have
  \begin{equation}
  \label{eq:theorem1}
      \min_{z \in \cup_i P'_{C_i}} f(z;\theta) \leq \min_{z \in \cup_i P^*_{C_i}} f(z;\theta)
  \end{equation}

On one hand, $P'_C = \cup_i \{P'_{C_i}\}$ has a score of $F'_C=\min_{z \in \cup_i P'_{C_i}} f(z;\theta)$. On the other hand, the algorithm by design chooses the higher scoring one between $P_C=\{C\}$ with a score of $F_C=f(C; \theta)$ and $P_C=\cup_i P^*_{C_i}$ with a score of $F_C=\min_{z \in \cup_i P^*_{C_i}} f(z;\theta)$, ensuring that $F_C \geq \min_{z \in \cup_i P^*_{C_i}} f(z;\theta)$. With these and~\eqref{eq:theorem1}, we conclude $F_C \geq F'_C$, which contradicts the assumption $F'_C > F_C$.
\end{proof}

\begin{algorithm}[t]
  \caption{Optimal worst-case segmentation
    \label{alg:optimal-worst}}
  \begin{algorithmic}[1]
    \Function{OptMinSeg}{$C, T_C$}
    \Statex{\textbf{return} a segmentation $P_C$ with a score of $F_C$}
    \State{$P_C \gets \{C\}$}
    \State{$F_C \gets f(C; \theta)$}
    \State{$N_C \gets$ set of $C$'s children nodes in $T_C$}
    \If{$N_C \neq \emptyset$}
    \For{$C_i$ in $N_C$}
    \State{$T_{C_i} \gets$ subtree of $T_C$ rooted at $C_i$}
    \State{$P_{C_i}, F_{C_i}$ = \Call{OptMinSeg}{$C_i, T_{C_i}$}}
    \If{$F_{C_i} \leq F_C$}
    \Return{$P_C, F_C$}
    \EndIf
    \EndFor
    \If {$\min_i F_{C_i} > F_C$}
    \State{$P_C \gets \cup_i P_{C_i}$}
    \State{$F_C \gets \min_i F_{C_i}$}
    \EndIf
    \EndIf
    \Return{$P_C, F_C$}
    \EndFunction
  \end{algorithmic}
\end{algorithm}

{\bf Generality:} Our analysis makes no assumptions about the objectness function $f(C;\theta)$ except the fact that it cannot be affected by the partitioning of other segments. In particular, this would allow objectness to depend on contextual arrangement of surrounding points outside $C$ - e.g., $f(C,X;\theta)$.

{\bf Efficiency:} Given points $X$ and a tree $T_X$ with $N$ leaf nodes, Algorithm~\ref{alg:optimal-worst} guarantees to return the optimal worst-case segmentation after visiting every node in the tree. In practice, it might not visit all nodes. Instead, it skips the rest of sub-trees whenever one sub-tree exhibits lower score than the coarse segmentation (line 9 in Algorithm~\ref{alg:optimal-worst}). The algorithm's complexity is linear in $N$ despite the fact that the search space is exponential in $N$.

\subsection{Average-case segmentation}
\emph{Average-case} segmentation scores a global segmentation as the \emph{average} objectness among its local segments:
\begin{equation}
  \label{eq:average}
  F_{\avg{}}(P_X; \theta) = \frac{1}{M} \sum_{i=1}^M f(C_i; \theta)
\end{equation}
where $P_X \in S_{X,T}$, $P_X=\{C_1,\ldots,C_M\}$, and $C_i \subset X$.
We define $P_{X,\avg{}}^*$ as an \emph{optimal average-case segmentation} if
\begin{equation}
  \label{eq:average}
  P_{X,\avg{}}^* = \argmax_{P_X\in S_{X,T}} {F_{\avg{}} (P_X; \theta)}
\end{equation}

It turns out that the problem of finding the optimal \emph{average-case} segmentation does not have optimal substructure, unlike \emph{worst-case} segmentation, meaning a locally optimal partitioning might no longer be optimal when considering global partitioning. Formally speaking, Lemma~\ref{lemma:min} no longer holds once $\min$ is changed to $\avg$.

Despite without optimal substructure, we apply a similar greedy searching algorithm. The main difference is how we aggregate local scores. Though greedily averaging local scores might lead to myopic decisions in certain situations~(Figure~\ref{fig:counter}), it performs quite well in practice (Section~\ref{sec:exp}).

\subsection{Learning the objectness function}
\label{sec:learn}

We have discussed segmentation algorithms under the assumption that we already have access to an objectness function $f(C;\theta)$, which predicts an objectness score for a given point cloud. We now introduce how to learn this function. Despite there has been a line of work that focuses on learning better representation, including Kd-networks~\cite{klokov2017escape}, PointCNN~\cite{li2018pointcnn}, EdgeConv~\cite{wang2019dynamic}, PointConv~\cite{wu2019pointconv}, just to name a few, we choose a simple PointNet++ to parameterize such an objectness function as a proof of concept. Below, we talk about how to learn a PointNet++ as a regressor to predict objectness score.

{\bf Ground truth objectness:} First, we must define regression target, i.e.\ ground truth objectness, of a given segment $C$. Suppose we have ground truth segmentation $P^{gt}=\{C^{gt}_1, \ldots, C^{gt}_{L}\}$, where $L$ is the number of ground truth segments. We can define $C$'s target objectness as the largest point IoU between itself and any ground truth segment~\eqref{eq:iou}.

\begin{equation}
\label{eq:iou}
  Objectness(C, P^{gt}) = \max_{l=1,\ldots,L} \frac{|C \cap C^{gt}_l|}{|C \cup C^{gt}_l|}
\end{equation}

Such a definition of objectness is only reasonable if points are uniformly distributed in space. In practice, 3D sensors (e.g. LiDAR) tend to produce denser points near the sensor. In consequence, the objectness will be heavily influenced by the partitioning of points closer to the sensor. For example, imagine two objects are segmented into one segment. Suppose one object has $n_1$ points and the other has $n_2$. If we use vanilla IoU as objectness, this segment would score $\frac{\max(n_1,n_2)}{n_1+n_2}$. When $n_1 \gg n_2$, the score could be really close to 1 despite it clearly introduces an under-segmentation error. To compensate such bias towards nearby objects, we propose a simple modification to IoU as in \eqref{eq:weighted-iou}.

\begin{equation}
\label{eq:weighted-iou}
  Objectness(C, P^{gt}) = \max_{l=1,\ldots,L} \frac{\sum_{x \in C \cap C^{gt}_l} x^T x  }{\sum_{x \in C \cup C^{gt}_l} x^T x }
\end{equation}
where $x^T x$ represents a point $x$'s squared distance to sensor origin. \eqref{eq:iou} is a special case, where $x^Tx$ is replaced with $1$.

\begin{figure}[t]
    \centering
    \includegraphics[width=\linewidth]{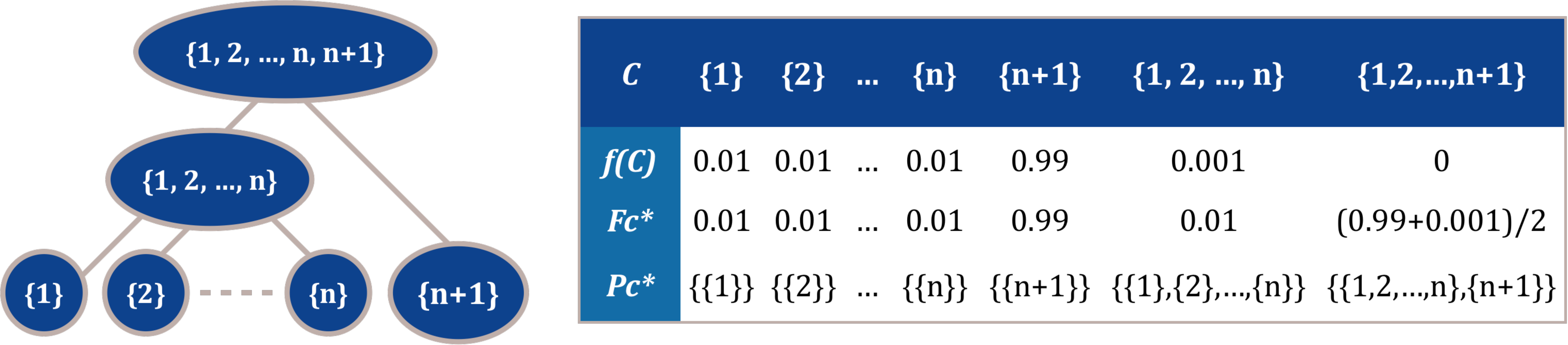}
    \caption{We illustrate why average-case segmentation does not have optimal substructure. We plot a tree on the left and show local objectness scores on the right. In this case, the optimal average-case segmentation of the root node, i.e. \{\{1,2,\ldots,n\},\{n+1\}\} cannot be formed by the optimal average-case segmentations of its children nodes, i.e. \{\{1\},\{2\},\ldots,\{n\}\} and \{\{n+1\}\}.}
    \label{fig:counter}
\end{figure}

{\bf Implementation:} We train a PointNet++ w/ multi-scale grouping (MSG)~\cite{qi2017pointnet++} for learning the objectness function. Starting from the off-the-shelf architecture, we replaced the classifier with a regressor that produces a real-value given an input point cloud. We applied a sigmoid function to convert the regression output to numbers between [0,1]. Finally, we compute the mean-squared error between prediction and ground truth objectness and perform backprop. In terms of preprocessing, we follow \cite{qi2018frustum} to make sure the input cloud is centered at origin and rotated based on the viewpoint. To facilitate batch processing, we follow the standard practice for PointNet++ and re-sample each segment to 1024 points.

\subsection{Building tree hierarchies}
\label{sec:tree}

We have discussed segmentation algorithms under the assumption that we have access to a tree hierarchy. Now we introduce how to build such a tree hierarchy given a set of points $X$. One natural approach is agglomerative clustering. After we define a metric (i.e.\ pairwise distance between two points) and a linkage criteria (i.e.\ pairwise distance between two sets of points), we can start from $\{\{x_1\},\ldots, \{x_N\}\}$ and keep merging the closest pair of point sets by taking the union over them, until all points are merged into one set. Such an approach produces a tree in a bottom-up fashion.

This approach tends to create tree hierarchies with very fine granularity, e.g.\ one node may differ from another with only one point of difference. As we have mentioned, our segmentation algorithms need to evaluate the objectness of every node in the tree. From an efficiency point of view, we would like to build a coarser tree whose leaf nodes are segments rather than individual points. Moreover, adjacent nodes should differ from each other much more.

{\bf Implementation:} We build tree hierarchies by applying Euclidean Clustering~\cite{rusu2010semantic} recursively in a top-down fashion with a list of decreasing $\epsilon$. Since Euclidean Clustering finds connected components w.r.t.\ a distance threshold $\epsilon$, we start with the largest $\epsilon$ that defines the most coarse connected components. Then, we apply Euclidean Clustering with a smaller $\epsilon$ within each connected component. This produces a multiple-tree top-down hierarchy. In our experiments, we use $\epsilon \in \{2m, 1m, 0.5m, 0.25m\}$ to build tree hierarchies for both training and testing. During training, we extract segments out of tree hierarchies built with the same parameters to form our training set for learning the objectness function. During testing, we apply the same learned objectness function in both worst-case semgentation and average-case segmentation.

\begin{figure*}
\centering
    (a)\hspace{.5\linewidth}(b)\\[.1em]
    \includegraphics[clip,trim={0 10cm 0 5cm},width=.496\linewidth]{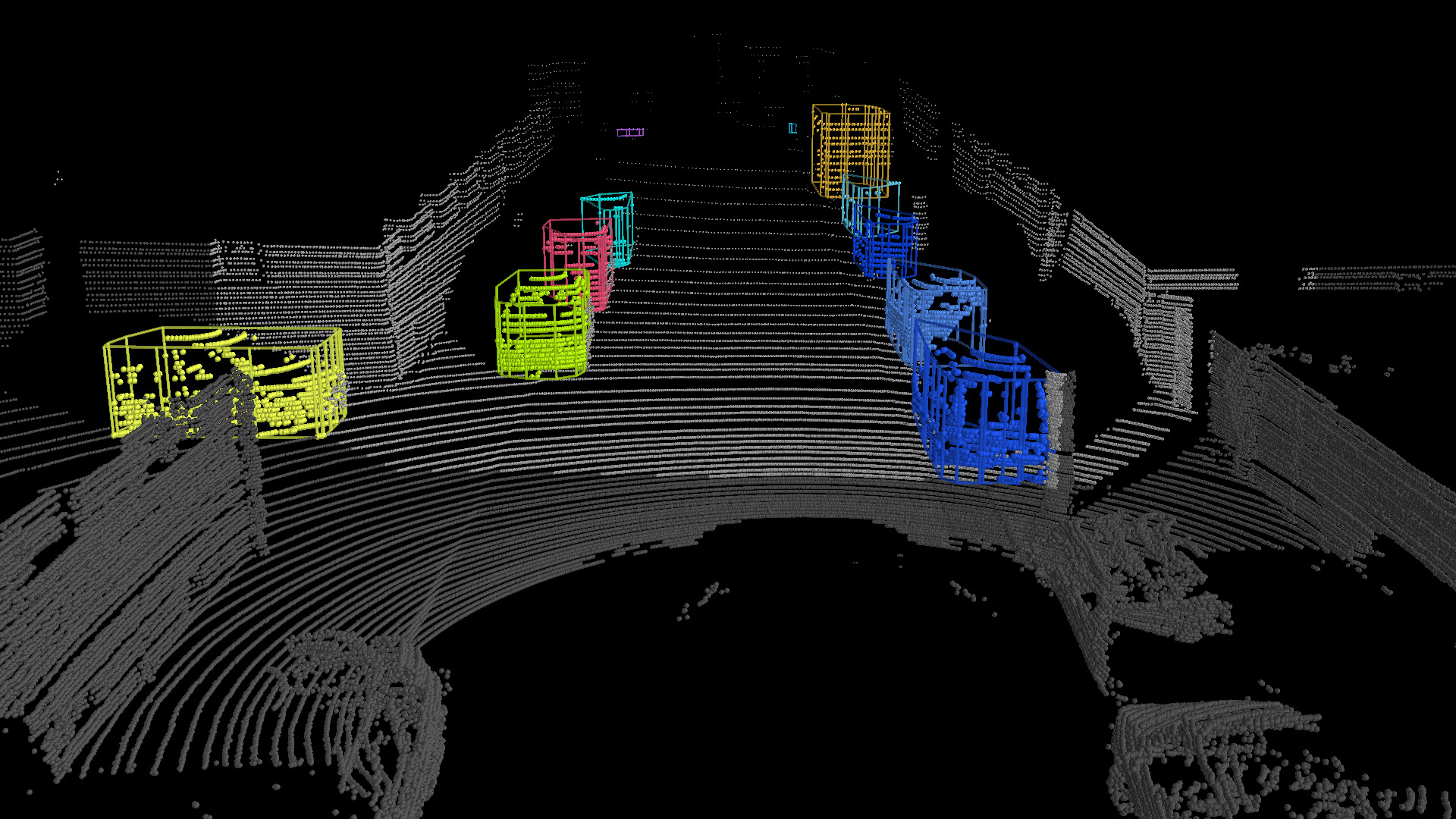}
    \includegraphics[clip,trim={0 10cm 0 5cm},width=.496\linewidth]{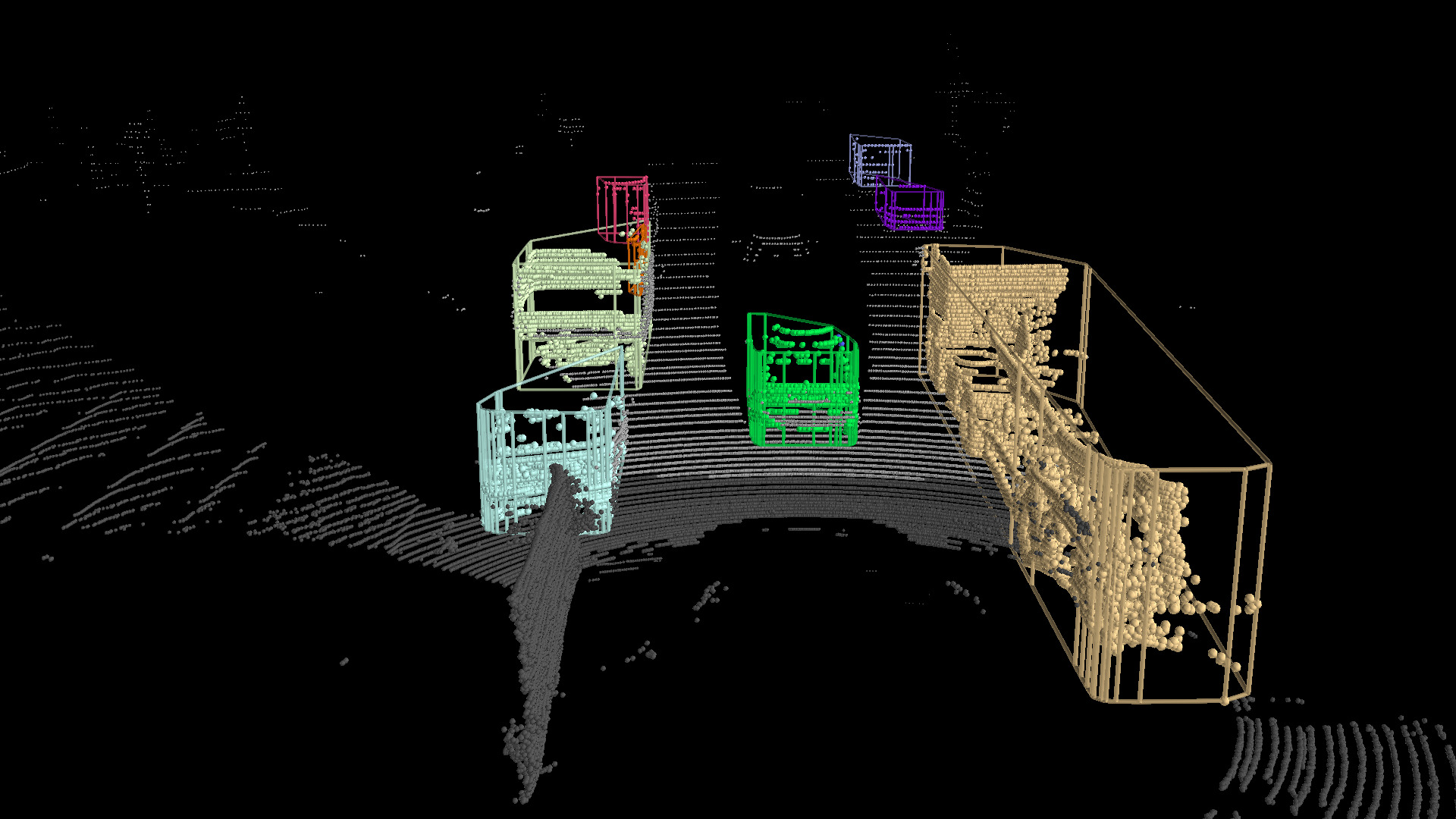}\\[.1em]
    \includegraphics[clip,trim={0 10cm 0 5cm},width=.496\linewidth]{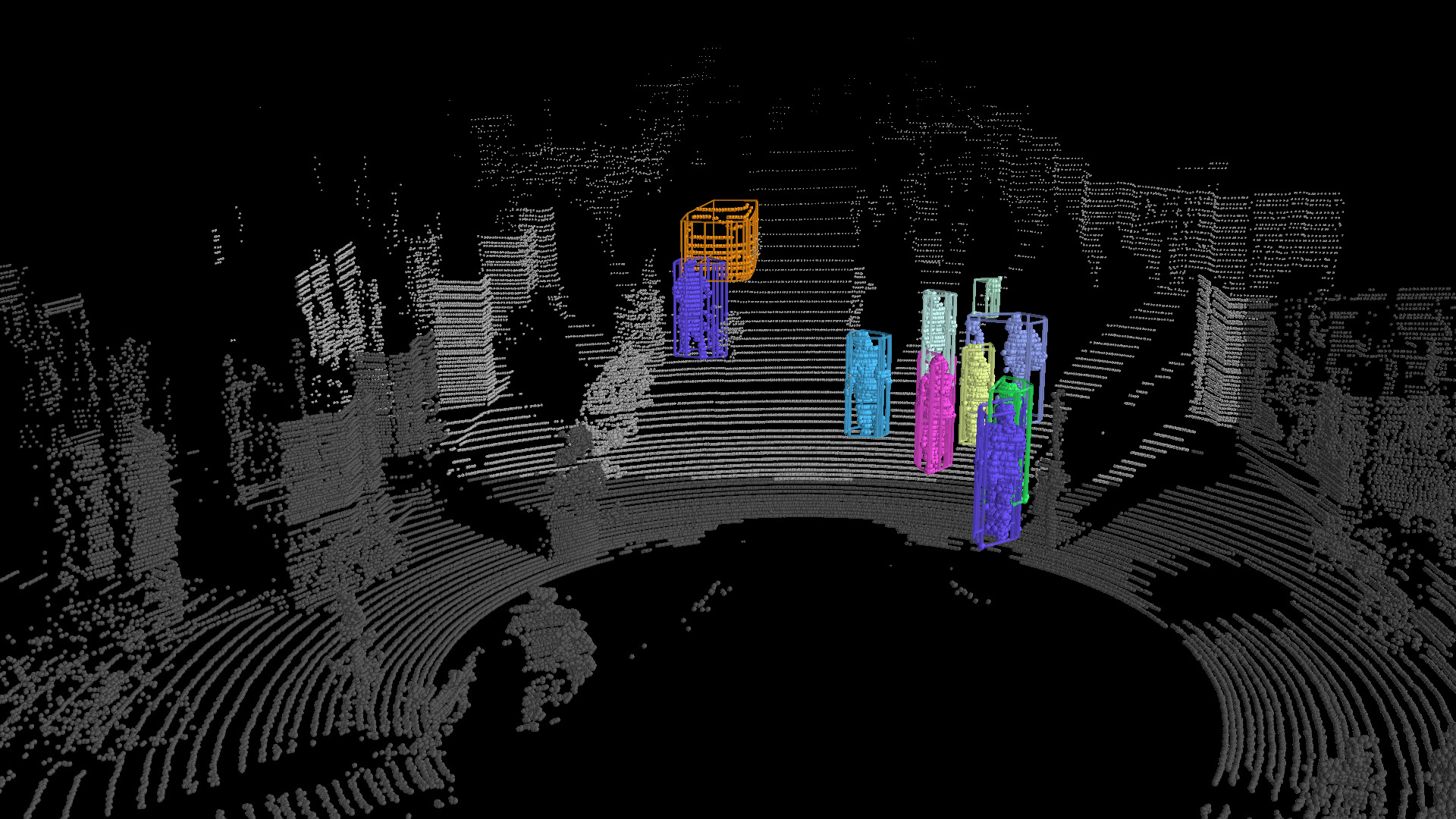}
    \includegraphics[clip,trim={0 10cm 0 5cm},width=.496\linewidth]{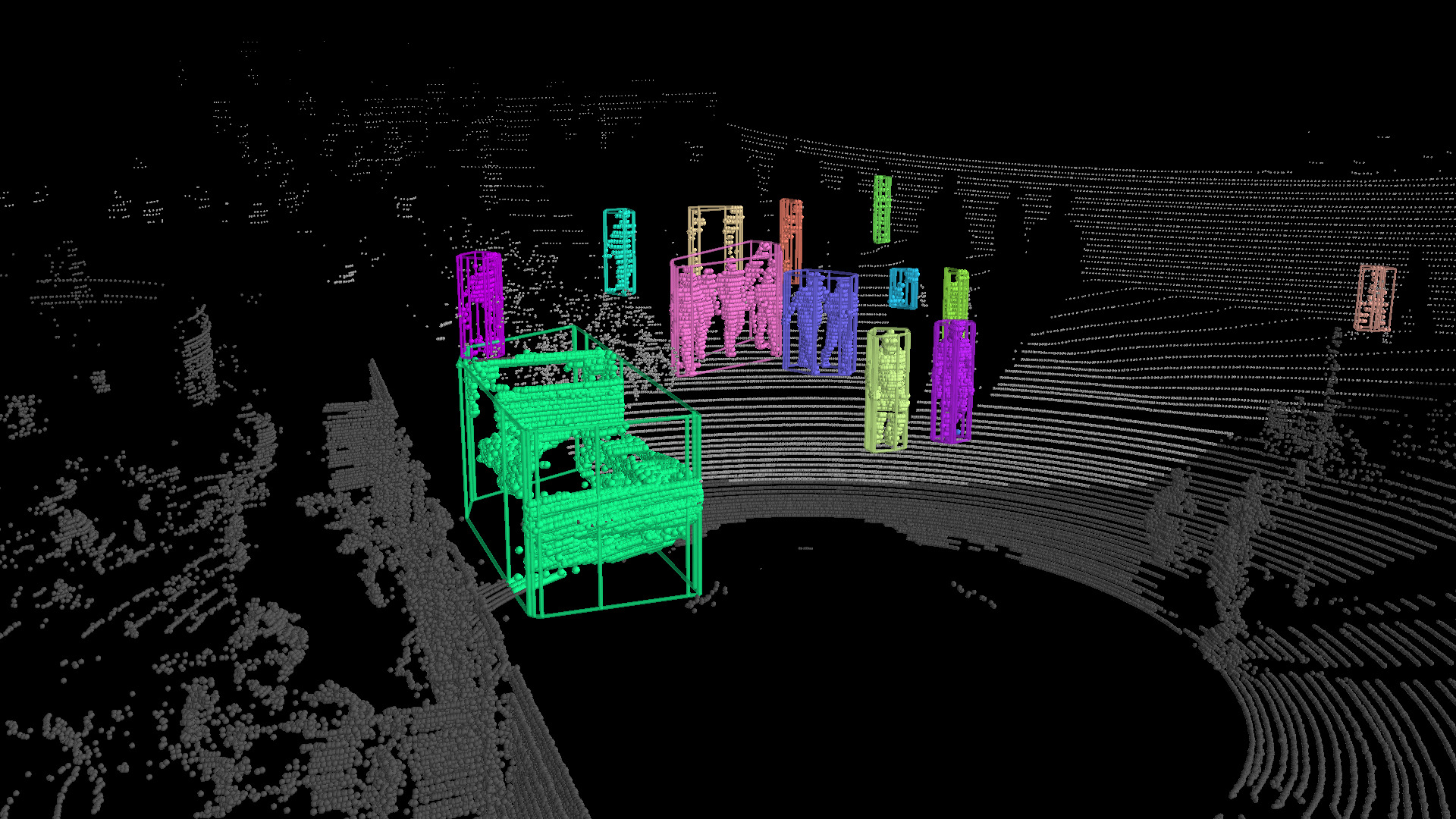}
    % \\ \vspace{-.3em}
    (c)\hspace{.5\linewidth}(d)
    \caption{We visualize more qualitative results of the proposed algorithm Ours(avg) on KITTI\@. In (a), we show a common scenario where there are parked cars on both sides of the road. In (b), we show a rare scenario where there is an oversized tank truck in the right lane. In (c), we show a scenario where a group of pedestrians walking in front of the autonomous vehicle. In (d), we show a typical failure case where pedestrians walk closely side by side. For such cases, there is often no perfect solution within the search space generated by EC.}
    \label{fig:examples}
\end{figure*}

\section{Experiments}
\label{sec:exp}

For evaluation, we repurpose the KITTI object detection benchmark for point cloud segmentation following the setup in~\cite{held2016probabilistic}. In our case, 3D objects do not physically overlap with one another. Therefore, we use ground truth 3D bounding boxes to produce ground truth segmentation. To do so, we first remove all points outside ground truth 3D bounding boxes (Figure~\ref{fig:splash}). Then we treat points within one ground truth 3D bounding box as the ground truth segment for the object. On KITTI, there exist ground truth 3D bounding boxes that overlap with each other. We ignore such segments during evaluation, since it is not clear how to define the ground-truth for the points in such bounding boxes~\cite{held2016probabilistic}. We follow \cite{chen2017multi} for splitting data into training and validation.

{\bf Evaluation protocol} We follow evaluation metrics introduced by Held et al.~\cite{held2016probabilistic}, which consists of two errors, under-segmentation error and over-segmentation error. Given ground truth segmentation $P^{gt}=\{C^{gt}_1,\ldots, C^{gt}_L\}$, we compute under-segmentation error $U$ and over-segmentation error $O$ given an output segmentation $P=\{C_1,\ldots, C_M\}$ as:
\begin{align}
    U &= \frac{1}{L} \sum_{l=1}^L \mathbf{1} ( \frac{|C_{i^*} \cap C_l^{gt}|}{|C_{i^*}|} < \tau_U ) \\
    O &= \frac{1}{L} \sum_{l=1}^L \mathbf{1} ( \frac{|C_{i^*} \cap C_l^{gt}|}{|C_l^{gt}|} < \tau_O )
\end{align}
with
\begin{equation}
  \label{eq:argmax}
  i^* = \argmax_{i=1}^M{|C_i \cap C^{gt}_l|}
\end{equation}
where $\mathbf{1}(\cdot)$ is an indicator function and $\tau_U, \tau_O$ are both constant thresholds. We set $\tau_U=2/3$ and $\tau_O=1$ following \cite{held2016probabilistic}. We ignore ground truth objects with overlapping bounding boxes ($529/20870\approx 2.5\%$) and those with 0 points ($238/20870\approx 1.1\%$) inside their 3D boxes. Other than these, we compute segmentation errors over all objects from all classes and also provide errors focusing on objects within 15m. We also adopt a slightly modified evaluation: instead of skipping objects with overlapping boxes entirely, we only ignore their overlapped regions.

\begin{table*}[t]
  \captionof{table}{\footnotesize \label{tab:all} Segmentation errors on KITTI Val. Left shows under-, over-segmentation, and total error. Right shows total error on a per-class basis. }
  \resizebox{\linewidth}{!}{
    \begin{tabular}{lrrrrrr|rrrrrrrrrrrrrrrrrr}
    \toprule
    \multirow{2}{*}{Method} & \multicolumn{2}{c}{under} & \multicolumn{2}{c}{over} & \multicolumn{2}{c|}{\textbf{total}} & \multicolumn{2}{c}{car} & \multicolumn{2}{c}{van} & \multicolumn{2}{c}{truck} & \multicolumn{2}{c}{pedestrian} & \multicolumn{2}{c}{person sitting} & \multicolumn{2}{c}{cyclist} & \multicolumn{2}{c}{\emph{tram}} & \multicolumn{2}{c}{misc} & \multicolumn{2}{c}{\textbf{mean}} \\
    \cmidrule(lr){2-3}\cmidrule(lr){4-5}\cmidrule(lr){6-7}\cmidrule(lr){8-9}\cmidrule(lr){10-11}\cmidrule(lr){12-13}\cmidrule(lr){14-15}\cmidrule(lr){16-17}\cmidrule(lr){18-19}\cmidrule(lr){20-21}\cmidrule(lr){22-23}\cmidrule(lr){24-25}
    & \multicolumn{1}{c}{all} & \multicolumn{1}{c}{15m} & \multicolumn{1}{c}{all} & \multicolumn{1}{c}{15m} & \multicolumn{1}{c}{all} & \multicolumn{1}{c|}{15m} & \multicolumn{1}{c}{all} & \multicolumn{1}{c}{15m} & \multicolumn{1}{c}{all} & \multicolumn{1}{c}{15m} & \multicolumn{1}{c}{all} & \multicolumn{1}{c}{15m} & \multicolumn{1}{c}{all} & \multicolumn{1}{c}{15m} & \multicolumn{1}{c}{all} & \multicolumn{1}{c}{15m} & \multicolumn{1}{c}{all} & \multicolumn{1}{c}{15m} & \multicolumn{1}{c}{all} & \multicolumn{1}{c}{15m} & \multicolumn{1}{c}{all} & \multicolumn{1}{c}{15m} & \multicolumn{1}{c}{all} & \multicolumn{1}{c}{15m} \\
    \midrule
    EC(2m) & 23.01 & 42.90 & 5.38 & 0.46 & 28.4 & 43.4 & 24.4 & 37.4 & 18.3 & 21.1 & 29.3 & \textbf{18.5} & 55.2 & 68.6 & 67.7 & 66.7 & 28.7 & 53.0 & \textbf{93.8} & 55.0 & 28.6 & 33.6 & 43.2 & 44.2 \\
    EC(1m) & 9.04 & 21.83 & 25.49 & 7.59 & 34.5 & 29.4 & 31.4 & 21.4 & 44.4 & 22.1 & 51.6 & 59.3 & 37.9 & 50.7 & 59.6 & 60.2 & 17.4 & 35.6 & 118.8 & 100.0 & 37.4 & 33.6 & 49.8 & 47.8 \\
    EC(0.5m) & 3.20 & 7.67 & 65.21 & 51.11 & 68.4 & 58.8 & 74.5 & 67.7 & 80.9 & 73.2 & 79.3 & 92.6 & 26.2 & 34.9 & \textbf{42.4} & \textbf{41.9} & 25.5 & 16.1 & 118.3 & 100.0 & 63.7 & 41.0 & 63.9 & 58.4 \\
    EC(0.25m) & 1.13 & 2.69 & 91.05 & 82.22 & 92.2 & 84.9 & 97.0 & 98.6 & 98.8 & 99.1 & 98.7 & 100.0 & 47.7 & 41.6 & 56.6 & 57.0 & 87.5 & 56.4 & 118.8 & 100.0 & 94.3 & 81.1 & 87.4 & 79.2 \\
    EC(all)* & \textit{7.83} & \textit{12.89} & \textit{5.38} & \textit{0.46} & \textit{13.2} & \textit{13.3} & \textit{10.9} & \textit{11.7} & \textit{13.6} & \textit{5.6} & \textit{29.3} & \textit{18.5} & \textit{14.3} & \textit{20.1} & \textit{27.3} & \textit{26.9} & \textit{10.4} & \textit{14.1} & \textit{93.8} & \textit{55.0} & \textit{13.6} & \textit{3.3} & \textit{26.6} & \textit{19.4} \\
    \midrule
    % AVOD & 25.22 & 14.63 & 57.20 & 68.87 & 82.4 & 83.5 & 81.9 & 85.7 & 86.4 & 64.8 & 78.8 & 29.6 & 82.1 & 87.0 & 93.9 & 93.5 & 88.3 & 87.9 & \textbf{56.7} & \textbf{20.0} & 86.4 & 63.1 & 81.8 & 66.5 \\
    AVOD & - & - & - & - & - & - & 81.9 & 85.7 & - & - & - & - & 82.1 & 87.0 & - & - & 88.3 & 87.9 & - & - & - & - & - & - \\
    % AVOD++ & 4.20 & 8.81 & 15.30 & 10.03 & 19.5 & 18.8 & 12.5 & 10.7 & 38.6 & 27.2 & 51.6 & 59.3 & 25.0 & 32.6 & 59.6 & 60.2 & 13.1 & 18.8 & 118.8 & 100.0 & 37.2 & 33.6 & 44.6 & 42.8 \\
    AVOD++ & - & - & - & - & - & - & 12.5 & 10.7 & - & - & - & - & 25.0 & 32.6 & - & - & 13.1 & 18.8 & - & - & - & - & - & - \\
    % PointPillars++ & 1.38 & 1.90 & 28.65 & 26.05 & 30.0 & 27.9 & 21.1 & 18.5 & 52.0 & 57.3 & 65.2 & 96.3 & 33.7 & 34.0 & 60.6 & 63.4 & 40.5 & 31.5 & 120.5 & 100.0 & 63.1 & 68.9 & 57.1 & 58.7 \\
    PointPillars++ & - & - & - & - & - & - & 21.1 & 18.5 & - & - & - & - & 33.7 & 34.0 & - & - & 40.5 & 31.5 & - & - & - & - & - & - \\
    % PointRCNN & 5.94 & 13.81 & 9.99 & 4.05 & 15.9 & 17.9 & 7.6 & 5.3 & 22.5 & 6.1 & 49.5 & 51.9 & 38.0 & 50.7 & 63.6 & 62.4 & 17.4 & 35.6 & 118.3 & 100.0 & 32.5 & 27.0 & 43.7 & 42.4 \\
    PointRCNN++ & - & - & - & - & - & - & 7.6 & 5.3 & - & - & - & - & - & - & - & - & - & - & - & - & - & - & - & - \\
    \midrule
    % SECOND++(4) & 2.78 & 4.98 & 11.52 & 7.29 & 14.3 & 12.3 & 7.9 & 3.9 & 22.4 & 7.5 & 48.4 & 66.7 & 22.7 & 29.8 & 58.6 & 59.1 & 10.6 & 14.1 & 118.8 & 100.0 & 33.0 & 23.8 & 40.3 & 38.1 \\
    SECOND++(4) & - & - & - & - & - & - & 7.9 & 3.9 & 22.4 & 7.5 & - & - & 22.7 & 29.8 & - & - & 10.6 & 14.1 & - & - & - & - & - & - \\
    % + Ext. Range & 2.68 & 4.49 & 10.48 & 7.56 & 13.2 & 12.0 & \textbf{7.0} & 4.3 & 18.3 & 7.0 & 44.2 & 55.6 & 23.2 & 29.5 & 55.6 & 55.9 & 10.0 & \textbf{10.1} & 116.5 & 100.0 & 30.8 & 20.5 & 38.2 & 35.4 \\
    + Ext. Range & - & - & - & - & - & - & \textbf{7.0} & 4.3 & 18.3 & 7.0 & - & - & 23.2 & 29.5 & - & - & 10.0 & \textbf{10.1} & - & - & - & - & - & - \\
    % + BG Removal & 2.04 & 2.58 & 13.38 & 8.18 & 15.4 & 10.8 & 9.7 & \textbf{3.7} & 26.0 & 8.9 & 46.1 & 55.6 & 21.5 & \textbf{26.2} & 36.4 & 36.6 & 11.9 & 15.4 & 118.3 & 100.0 & 26.8 & 16.4 & 37.1 & 32.9 \\
    + BG Removal & - & - & - & - & - & - & 9.7 & \textbf{3.7} & 26.0 & 8.9 & - & - & 21.5 & \textbf{26.2} & - & - & 11.9 & 15.4 & - & - & - & - & - & - \\
    % + Both & 2.12 & 2.61 & 12.19 & 8.13 & 14.3 & \textbf{10.7} & 9.1 & \textbf{3.7} & 20.6 & 8.0 & 41.9 & 55.6 & 21.5 & 26.5 & \textbf{35.4} & \textbf{35.5} & 12.0 & 14.1 & 115.6 & 100.0 & 24.6 & 18.0 & 35.1 & 32.7 \\
    + Both & - & - & - & - & - & - & 9.1 & \textbf{3.7} & 20.6 & 8.0 & - & - & 21.5 & 26.5 & - & - & 12.0 & 14.1 & - & - & - & - & - & - \\
    \midrule
    SECOND++(8) & 2.52 & 3.89 & 11.39 & 8.02 & 13.9 & 11.9 & 7.9 & 4.9 & 23.2 & 8.0 & 43.2 & 51.9 & 22.3 & 26.8 & 54.5 & 54.8 & 9.6 & 13.4 & 117.9 & 100.0 & 26.1 & 15.6 & 38.1 & 34.4 \\
    + Ext. Range & 2.59 & 4.27 & 10.12 & 8.13 & \textbf{12.7} & 12.4 & 7.2 & 4.7 & 17.6 & 8.5 & 34.0 & 48.1 & 23.2 & 29.5 & 57.6 & 58.1 & \textbf{9.5} & 12.1 & 112.9 & 100.0 & 25.0 & 18.9 & 35.9 & 35.0 \\
    + BG Removal & 2.30 & 3.15 & 13.03 & 8.73 & 15.3 & 11.9 & 9.8 & 4.8 & 24.8 & 5.6 & 45.8 & 63.0 & 22.0 & 28.5 & 44.4 & 44.1 & 11.7 & 14.8 & 114.7 & 95.0 & 23.5 & 17.2 & 37.1 & 34.1 \\
    + Both & 2.24 & 3.07 & 11.94 & 8.65 & 14.2 & \textbf{11.7} & 9.2 & 4.5 & 20.1 & 6.1 & 38.2 & 63.0 & 22.5 & 28.5 & 44.4 & 44.1 & 11.5 & 13.4 & 111.2 & 95.0 & 20.4 & 18.0 & 34.7 & 34.1 \\
    \midrule
    Ours(min) & 13.13 & 21.42 & 5.65 & 0.60 & 18.8 & 22.0 & 15.8 & 18.0 & 15.5 & 11.3 & 29.3 & \textbf{18.5} & 27.5 & 35.6 & 53.5 & 51.6 & 17.9 & 26.2 & 93.8 & 55.0 & 20.7 & \textbf{11.5} & 34.3 & 28.5 \\
    Ours(avg) & 8.64 & 12.75 & 7.89 & 4.73 & 16.5 & 17.5 & 13.7 & 14.7 & 14.9 & 7.0 & 30.1 & 29.6 & \textbf{20.9} & \textbf{26.2} & \textbf{42.4} & \textbf{41.9} & 16.4 & 19.5 & 94.2 & 55.0 & 20.0 & \textbf{11.5} & 31.6 & 25.7 \\
    \midrule
    Ours(avg) w/ \\
    (2.7, 0.9, 0.3)m & 11.49 & 15.14 & 6.17 & 5.22 & 17.7 & 20.4 & 15.8 & 17.3 & 13.9 & 13.6 & 23.6 & 25.9 & 23.6 & 29.1 & 47.5 & 47.3 & 18.3 & 26.2 & 70.5 & 20.0 & 17.4 & 15.6 & 28.8 & 24.4 \\
    (2.4, 1.2, 0.6, 0.3)m & 9.30 & 11.96 & 5.93 & 5.90 & 15.2 & 17.9 & 12.8 & 15.1 & 11.6 & 7.5 & 27.2 & 29.6 & 21.8 & 26.7 & \textbf{42.4} & \textbf{41.9} & 15.7 & 19.5 & 79.5 & 25.0 & \textbf{16.0} & 16.4 & 28.4 & 22.7 \\
    (3.2, 1.6, 0.8, 0.4, 0.2)m & 10.53 & 10.60 & 4.52 & 5.36 & 15.1 & 16.0 & 12.8 & 12.8 & \textbf{10.8} & \textbf{4.2} & \textbf{21.5} & 22.2 & 23.3 & 27.4 & \textbf{42.4} & \textbf{41.9} & 17.4 & 16.8 & \textbf{66.1} & \textbf{10.0} & 16.9 & 13.9 & \textbf{26.4} & \textbf{18.7}\\
    \bottomrule
    \end{tabular}
  }
\end{table*}

\subsection{Baselines}
\label{sec:baselines}

{\bf Euclidean clustering:} We use Euclidean clustering with 4 different distance threshold $\{2m, 1m, 0.5, 0.25m\}$ to build trees of segments, which defines the space of possible segmentations for our approach. Therefore, it makes sense to include all 4 of them as baselines and see if our approach indeed finds a better solution.% than the best of them.

{\bf State-of-the-art 3D detectors:} We compare our approach to AVOD~\cite{ku2018joint}, PointPillars~\cite{lang2018pointpillars}, PointRCNN~\cite{shi2019pointrcnn}, and SECOND~\cite{yan2018second}. We follow the off-the-shelf training and testing setting as closely as possible. For AVOD, we re-train a LiDAR-only car detector and a LiDAR-only people (pedestrian and cyclist) detector following official implementation\footnote{\url{https://github.com/kujason/avod}}. For PointPillars, we re-train a detector that simultaneously detects cars and people (pedestrian and cyclist) following an author-endorsed implementation\footnote{\url{https://github.com/traveller59/second.pytorch}}. For PointRCNN, we evaluate the official pre-trained car model as there are no available ones for other classes within its official implementation\footnote{\url{https://github.com/sshaoshuai/pointrcnn}}. For SECOND, since it is our best performing baseline, besides re-training the off-the-shelf model, we also explore various ways to improve its performance. By design, these detectors output class-specific bounding box detection. To produce class-agnostic segmentations, we ignore the class label and follow a greedy procedure: We start with the highest scoring bounding box and group all points within the box as one segment. We then remove those points and move onto the next highest scoring detection. We repeat until exhausting either detections or 3D points. In the end, we might still not have every point assigned to a segment. A simple fix is grouping leftover points as a new segment. We discuss a much better alternative approach below.

{\bf Detector++:} A better approach to handling missed detection is to fall back to clustering. Specifically, we apply Euclidean Clustering (EC) with a fixed $\epsilon$ on all leftover points, producing a set of leftover segments. For each leftover segment, we check if it can merged into an existing detection segment, using the criteria of whether the smallest pairwise distance between two segments is smaller than the threshold $\epsilon$. If so, we merge the leftover segment into the detection segment. We refer to such baselines as Detector++ (e.g. AVOD++ etc.).

\subsection{Results}
\label{sec:results}

We first present qualitative examples of our approach segmenting rare objects on KITTI Val, as shown in Figure~\ref{fig:examples}. For quantitative evaluation, we present both per-class and overall segmentation errors in Table~\ref{tab:all}.

{\bf Ours(min) vs. Ours(avg):} We label the optimal worst-case segmentation as Ours(min) and the average-case segmentation as Ours(avg). Ours(avg) consistently outperforms Ours(min) in terms of the total error. Ours(min) produces a much lower over-segmentation error but a much higher under-segmentation error, suggesting it makes more mistakes of grouping different objects into one segment and less mistakes of splitting points from one single object into multiple segments. The cause of such behavior might be due to the risk-averse objective of optimal worst-case segmentation. However, current evaluation does not emphasize the worst-case performance, instead, it measures the average performance over all objects. We observe that if we evaluate the worst-case objectness (Section~\ref{sec:additional-eval}), Ours(min) does outperform both Ours(avg) and AVOD++.

{\bf Ours vs. Euclidean Clustering:} We label Euclidean Clustering as ``EC($\epsilon$)'', where $\epsilon$ represents the distance threshold (meter). All together, they define a segment hierarchy. We construct a pool of segments that contains every node (segment) in the hierarchy and call this ``EC(all)*''. This serves as a \emph{unreachable} upper-bound, since segments from such a pool overlap with each other, which violates the non-disjoint constraint of a valid partition. Nonetheless, it shows that there gap between our proposed method and the upper bound is relatively small (3-4\%), suggesting plenty of room left for improvement in creating better hierarchies.

{\bf Detector++ vs. Detector:} We focus on AVOD to demonstrate the improvement of Detector++ over Detector. AVOD produces much larger oversegmentation errors, likely due to imprecisely localized 3D bounding boxes. For example, when a 3D bounding box is predicted smaller than it should be, the resultant segment might miss points on the edge, leading to oversegmentation. AVOD++ is designed to fix this issue and dramatically improves the oversegmentation error. The undersegmentation errors also improves significantly from AVOD to AVOD++, likely due to successfully segmenting objects that are completely missed by detections.

{\bf Ours vs. Detector++:} SECOND++ performs the best among all Detector++ baselines and also achieves the lowest overall total error among all methods. However, if we break down total segmentation errors on a per-class basis, our approaches perform much better than SECOND++. Such difference is due to a skewed data distribution. For example, 68\% objects are labeled as car while only 3\% are labeled as misc. SECOND++ performs better on common classes such as car and ours perform better on rare ones such as misc.

{\bf Runtime analysis:} Our algorithm requires running PointNet++ on every candidate segment in order to compute its objectness. In practice, one frame from KITTI Val, which contains $68 (\sigma=42)$ segments on average, takes about $0.19s (\sigma=0.06s)$ to process on a single GTX 1080.

\subsection{Additional evaluation protocols}
\label{sec:additional-eval}

\begin{table}[t]
  \centering
  \caption{Instance segmentation AP[@.5:.95:.05] on KITTI Val.}
  \label{tab:instance}
  \resizebox{\linewidth}{!}
  {
  \begin{tabular}{lcccccccccc}
    \toprule
    & car & van & trk & ped & psit & cyc & tram & misc & \textbf{mean}\\
    \midrule
    % AVOD(LiDAR) & 64.4 & 10.3 & 2.5 & 31.1 & 0.1 & 15.5 & 1.2 & 1.0 & 15.8 \\
    AVOD & 64.4 & - & - & 31.1 & - & 15.5 & - & - & - \\
    % AVOD++ & 91.6 & 38.6 & 8.7 & 51.6 & 0.4 & 41.6 & 1.3 & 10.2 & 30.5 \\
    AVOD++ & 91.6 & - & - & 51.6 & - & 41.6 & - & - & - \\
    % PointPillars++ & 91.4 & 36.3 & 3.6 & 55.2 & 1.7 & 55.9 & 0.8 & 8.5 & 31.7 \\
    PointPillars++ & 91.4 & - & - & 55.2 & - & 55.9 & - & - & - \\
    % PointRCNN++ & 95.2 & 68.7 & 14.5 & 15.8 & 0.3 & 13.8 & 2.7 & 21.0 & 29.0 \\
    PointRCNN++ & 95.2 & - & - & - & - & - & - & - & - \\
    \midrule
    % SECOND++(4) & 95.1 & 68.9 & 19.4 & 68.6 & 2.0 & 65.9 & 3.3 & 20.8 & 43.0 \\
    SECOND++(4) & 95.1 & 68.9 & - & 68.6 & - & 65.9 & - & - & - \\
    % + Range & 95.8 & 75.4 & 26.0 & 70.2 & 2.1 & 68.1 & 6.5 & 24.5 & 46.1 \\
    + Ext. Range & 95.8 & 75.4 & - & 70.2 & - & 68.1 & - & - & - \\
    % + BG Removal & 95.3 & 74.0 & 20.9 & 77.5 & 12.5 & 71.8 & 6.0 & 36.7 & 49.3 \\
    + BG Removal & 95.3 & 74.0 & - & 77.5 & - & 71.8 & - & - & - \\
    % + Both & \textbf{96.0} & \textbf{82.0} & 39.6 & \textbf{78.1} & 12.2 & \textbf{72.9} & 9.1 & 42.7 & 54.1 \\
    + Both & \textbf{96.0} & \textbf{82.0} & - & \textbf{78.1} & - & \textbf{72.9} & - & - & - \\
    \midrule
    SECOND++(8) & 95.3 & 70.3 & 30.0 & 71.8 & 2.6 & 69.6 & 10.2 & 33.9 & 48.0 \\
    + Ext. Range & 95.9 & 78.3 & \textbf{63.4} & 71.0 & 2.9 & 71.9 & 13.4 & 39.6 & 54.5 \\
    + BG Removal & 95.1 & 73.5 & 30.3 & 76.2 & 9.0 & 71.4 & 13.1 & 47.3 & 52.0 \\
    + Both & \textbf{96.0} & 81.3 & 61.7 & 76.5 & 8.6 & 72.4 & 16.4 & \textbf{55.4} & \textbf{58.5} \\
    \midrule
    Ours(min) & 86.0 & 80.4 & 61.6 & 62.3 & 12.9 & 66.3 & \textbf{21.9} & 53.0 & 55.6 \\
    Ours(avg) & 89.8 & 81.1 & 58.6 & 69.2 & \textbf{14.0} & 68.2 & 19.8 & 51.0 & 56.5 \\
    \midrule
    Ours(avg) w/ \\
    (2.7, 0.9, 0.3)m & 87.5 & 78.6 & 57.6 & 66.7 & \textbf{14.0} & 66.9 & 20.8 & 49.7 & 55.2\\
    (2.4, 1.2, 0.6, 0.3)m & 89.6 & 81.9 & 59.4 & 67.9 & 13.7 & 69.2 & 21.5 & 52.1 & 56.9\\
    (3.2, 1.6, 0.8, 0.4, 0.2)m & 89.0 & 79.0 & 56.3 & 67.6 & 13.2 & 67.2 & 18.7 & 49.0 & 55.0\\
    % \midrule
    % Hybrid & 92.5 & 82.6 & 58.6 & 73.5 & 14.0 & 70.0 & 19.9 & 51.7 & 57.9 \\
    \bottomrule
  \end{tabular}
  }
\end{table}

{\bf Class-agnostic instance segmentation:} The evaluation protocol we adopt comes from the robotics community~\cite{held2014combining}. It differs from the standard evaluation in computer vision, i.e.\ per-voxel instance segmentation in ScanNet~\cite{dai2017scannet}. One key difference is that 3D instance segmentation does not require the output segmentation to be a valid partition. Instead, it treats the task as retrieval and evaluates the tradeoff between precision and recall. Here we take a similar approach as ScanNet, but modify the evaluation protocol to be class-agnostic and per-point instead of per-voxel.

As we can see in Table~\ref{tab:instance}, the observations are consistent with what we see in Table~\ref{tab:all}: SECOND++(8) with both modifications outperforms our segmentation approach on common classes such as \textit{car}, but falls short on rarer classes (such as person sitting and tram) by a large margin. Overall, the best SECOND approach outperforms the best variant of our approach by 1.6\% in mAP.

\begin{figure}
    \centering
    \includegraphics[width=.8\linewidth]{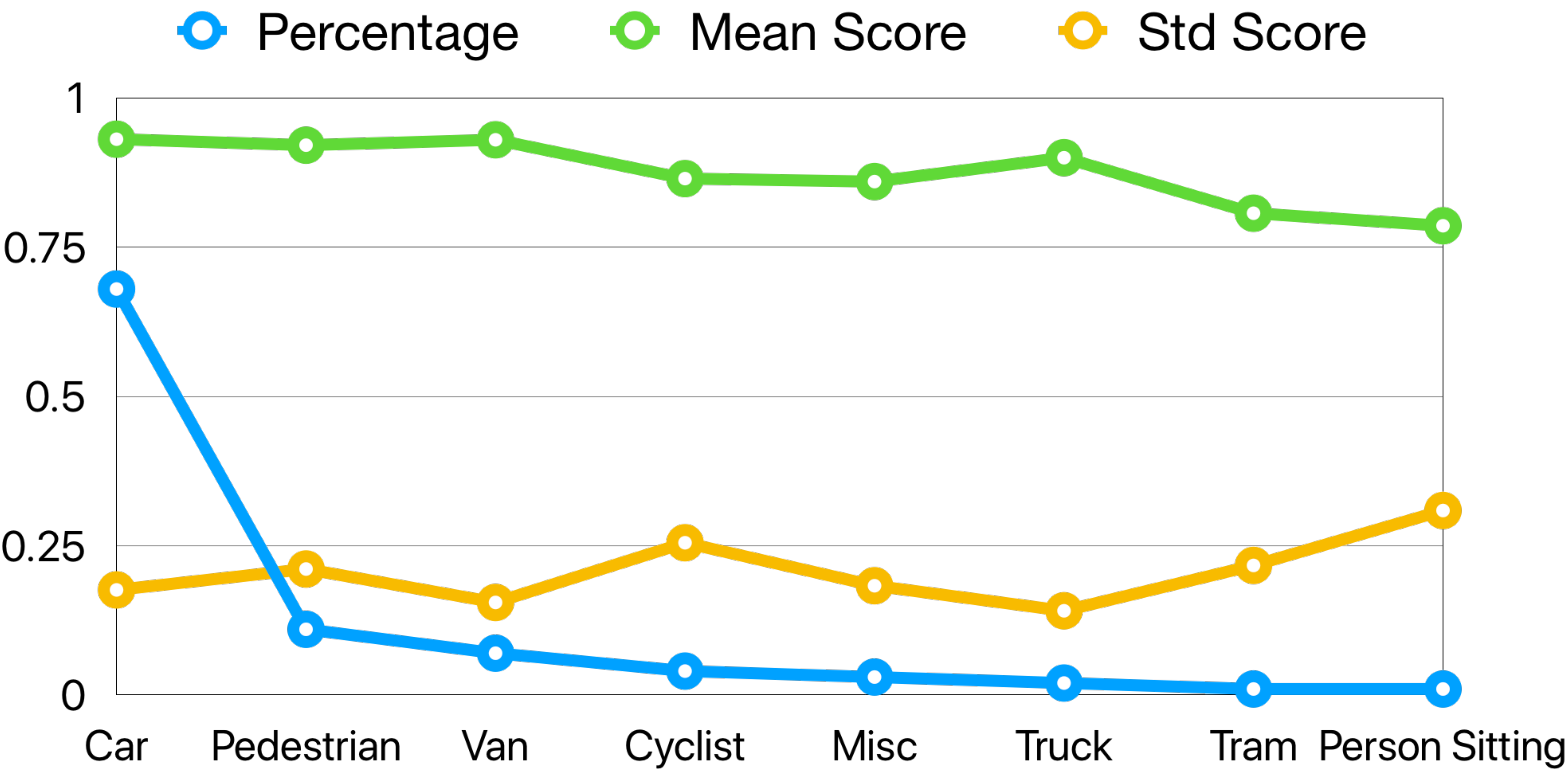}
    \caption{How the learned objectness model generalizes in the tail.}
    \label{fig:objectness}
\end{figure}

{\bf How objectness generalizes} To evaluate how well our learned objectness model generalizes, we apply it onto ground truth segments from the validation set. In Figure~\ref{fig:objectness}, we plot the average objectness score for each class and the standard deviation. We also show the percentage of objects for each class within the training set. As the number of training data decreases dramatically, the average score tends to drops slightly and the variance tends to rise slightly.

{\bf Worst-case evaluation} In Table~\ref{tab:all} and~\ref{tab:instance}, we see Ours(avg) outperforms Ours(min) despite the latter is provably optimal. We have briefly discussed the reason: current protocols do not evaluate worst-case performance. Here, we score the worst IoU between a set of local segments and the ground truths, as Eq.~\eqref{eq:worst-eval} shows, where $\{P_1\dots P_N\}$ and $\{P_1^{gt}\dots P_{N}^{gt}\}$ represents predicted and ground truth segmentation in each of the $N$ frames. We found Ours(min) scores a mean-worst IoU of 72.2, 4.2\% higher than Ours(avg).
\begin{equation}
  \label{eq:worst-eval}
  \sum_{i=1}^N \frac{1}{N} \min_{C \in P_i} \max_{C^{gt} \in P_i^{gt}} \frac{|C \cap C^{gt}|}{|C \cup C^{gt}|}
\end{equation}

\subsection{Additional diagnostics}
\label{sec:diagnostics}
{\bf Sensitivity analysis} Our objectness function is learned on segments from a EC hierarchy generated with 4 distance thresholds \{2m, 1m, 0.5m, 0.25m\}. To analyze how robust our algorithm is to change of hyper-parameters, we test the learned objectness function on different hierarchies. In Table~\ref{tab:all},~\ref{tab:instance}, we find that having a deeper hierarchy significantly reduces segmentation errors. Comparing to hard-thresholded segmentation errors, there are only slight changes in multi-threshold instance segmentation mAP.

{\bf Weighted vs.\ vanilla IoU} Here, we empirically compare weighted IoU and vanilla IoU in terms of defining the training target for our objectness model. As we see in Table~\ref{tab:objectness}, for both worst-case and average-case segmentation, the objectness model trained with weighted IoU perform slightly better than the one trained with vanilla IoU. Note ``Ours(min) - vanilla'' and ``Ours(avg) - vanilla'' share the exact same underlying objectness model.

\begin{table}
  \centering
  \caption{Segmentation errors on KITTI Val.}
  \label{tab:objectness}
  \begin{tabular}{lrrrrrr}
    \toprule
    \multirow{2}{*}{Method} & \multicolumn{2}{c}{Under (\%)} & \multicolumn{2}{c}{Over (\%)} & \multicolumn{2}{c}{Total (\%)}\\
    \cmidrule(lr){2-3}\cmidrule(lr){4-5}\cmidrule(lr){6-7} & \multicolumn{1}{c}{all} & \multicolumn{1}{c}{15m} & \multicolumn{1}{c}{all} & \multicolumn{1}{c}{15m} & \multicolumn{1}{c}{all} & \multicolumn{1}{c}{15m}\\
    \midrule
    Ours(min) - vanilla  & 13.91 & 22.40 & 5.58 & 0.60 & 19.5 & 23.0\\
    Ours(min) - weighted & 13.13 & 21.42 & 5.65 & 0.60 & \textbf{18.8} & \textbf{22.0} \\
    \midrule
    Ours(avg) - vanilla & 10.30 & 15.44 & 7.11 & 3.13 & 17.4 & 18.6\\
    Ours(avg) - weighted & 8.64 & 12.75 & 7.89 & 4.73 & \textbf{16.5} & \textbf{17.5}\\
    \bottomrule
  \end{tabular}
\end{table}

\section*{Conclusion}
We present an approach for class-agnostic point cloud segmentation.The approach efficiently searches over an exponentially large space of candidate segmentations and return one where individual segments score well according to a data-driven point-based model of ``objectness''. We prove that our algorithm is guaranteed to achieve optimality to a specific definition. On KITTI, we demonstrate our approach significantly outperforms past bottom-up approaches and top-down object-based algorithms for segmenting point clouds.

{\bf Acknowledgements:} This work was supported by the CMU Argo AI Center for Autonomous Vehicle Research.

\appendix

\setcounter{section}{0}
\setcounter{figure}{0}
\setcounter{footnote}{0}
\renewcommand{\thesection}{A\arabic{section}}
\renewcommand{\thefigure}{\Alph{figure}}
\def\thesection{\Alph{section}}

{\bf Slides} Please find a slide deck (\href{https://www.cs.cmu.edu/~peiyunh/seg/slides.mp4}{here}) that illustrates the main ideas in this paper.

{\bf Additional visualization} Please find videos (\href{https://www.cs.cmu.edu/~peiyunh/seg/kitti/2011_09_26_drive_0009.mp4}{1}, \href{https://www.cs.cmu.edu/~peiyunh/seg/kitti/2011_09_26_drive_0035.mp4}{2}, \href{https://www.cs.cmu.edu/~peiyunh/seg/kitti/2011_09_28_drive_0021.mp4}{3}) that show advantages and limitations of our approach.

{\bf Additional evaluation} In Table~\ref{tab:all-only-ignore-ovlp-region}, we show segmentation errors under a slightly modified evaluation: instead of skipping overlapping objects entirely, we only ignore the points that fall into the overlapping region.

\begin{table*}[t]
  \captionof{table}{\footnotesize \label{tab:all-only-ignore-ovlp-region} Segmentation errors on KITTI Val. Left shows under-, over-segmentation, and total error. Right shows total error on a per-class basis. Comparing to Table~\ref{tab:all}, instead of skipping overlapping objects entirely, we only ignore the points that fall into the overlapping region. %overlapping ground truth objects are handled in a different way. Instead of ignoring objects with overlap entirely, here we only ignore points that fall into the overlapping regions. We compute segmentation errors based on the remaining points.
  }
  \resizebox{\linewidth}{!}{
    \begin{tabular}{lrrrrrr|rrrrrrrrrrrrrrrrrr}
    \toprule
    \multirow{2}{*}{Method} & \multicolumn{2}{c}{under} & \multicolumn{2}{c}{over} & \multicolumn{2}{c|}{\textbf{total}} & \multicolumn{2}{c}{car} & \multicolumn{2}{c}{van} & \multicolumn{2}{c}{truck} & \multicolumn{2}{c}{pedestrian} & \multicolumn{2}{c}{person sitting} & \multicolumn{2}{c}{cyclist} & \multicolumn{2}{c}{\emph{tram}} & \multicolumn{2}{c}{misc} & \multicolumn{2}{c}{\textbf{mean}} \\
    \cmidrule(lr){2-3}\cmidrule(lr){4-5}\cmidrule(lr){6-7}\cmidrule(lr){8-9}\cmidrule(lr){10-11}\cmidrule(lr){12-13}\cmidrule(lr){14-15}\cmidrule(lr){16-17}\cmidrule(lr){18-19}\cmidrule(lr){20-21}\cmidrule(lr){22-23}\cmidrule(lr){24-25}
    & \multicolumn{1}{c}{all} & \multicolumn{1}{c}{15m} & \multicolumn{1}{c}{all} & \multicolumn{1}{c}{15m} & \multicolumn{1}{c}{all} & \multicolumn{1}{c|}{15m} & \multicolumn{1}{c}{all} & \multicolumn{1}{c}{15m} & \multicolumn{1}{c}{all} & \multicolumn{1}{c}{15m} & \multicolumn{1}{c}{all} & \multicolumn{1}{c}{15m} & \multicolumn{1}{c}{all} & \multicolumn{1}{c}{15m} & \multicolumn{1}{c}{all} & \multicolumn{1}{c}{15m} & \multicolumn{1}{c}{all} & \multicolumn{1}{c}{15m} & \multicolumn{1}{c}{all} & \multicolumn{1}{c}{15m} & \multicolumn{1}{c}{all} & \multicolumn{1}{c}{15m} & \multicolumn{1}{c}{all} & \multicolumn{1}{c}{15m} \\
    \midrule
    EC(2m) & 24.89 & 46.21 & 5.24 & 0.43 & 30.1 & 46.6 & 24.4 & 37.3 & 18.2 & 21.1 & \textbf{29.3} & \textbf{18.5} & 63.1 & 75.2 & 79.2 & 78.5 & 28.7 & 53.0 & 94.9 & 55.0 & 31.0 & 36.2 & 46.1 & 46.9 \\
    EC(1m) & 11.26 & 26.31 & 24.89 & 7.09 & 36.1 & 33.4 & 31.4 & 21.2 & 44.5 & 22.1 & 51.6 & 59.3 & 48.9 & 60.7 & 74.0 & 74.3 & 17.4 & 35.6 & 121.2 & 100.0 & 39.5 & 36.2 & 53.6 & 51.2 \\
    EC(0.5m) & 5.54 & 12.89 & 63.70 & 48.32 & 69.2 & 61.2 & 74.6 & 67.9 & 81.0 & 73.2 & 79.3 & 92.6 & 39.0 & 47.6 & 63.0 & 62.5 & 25.5 & 16.1 & 121.2 & 100.0 & 64.9 & 43.3 & 68.6 & 62.9 \\
    EC(0.25m) & 3.02 & 7.47 & 89.61 & 78.70 & 92.6 & 86.2 & 97.1 & 98.8 & 98.8 & 99.1 & 98.7 & 100.0 & 58.9 & 54.1 & 74.7 & 75.0 & 87.5 & 56.4 & 119.1 & 100.0 & 94.1 & 82.7 & 91.1 & 83.3 \\
    % AVOD++ & 7.22 & 16.22 & 13.63 & 7.68 & 20.8 & 23.9 & 11.3 & 10.2 & 36.2 & 21.1 & 50.5 & 59.3 & 41.5 & 49.9 & 77.3 & 76.4 & 13.9 & 22.8 & 121.2 & 100.0 & 39.0 & 36.2 & 48.9 & 47.0 \\
    % PointPillars (separate) & 2.14 & 3.38 & 29.51 & 28.32 & 31.6 & 31.7 & 21.4 & 19.1 & 52.2 & 57.3 & 65.2 & 96.3 & 42.4 & 43.2 & 79.2 & 81.9 & 40.5 & 31.5 & 121.2 & 100.0 & 64.2 & 70.1 & 60.8 & 62.4 \\
    EC(all)* & \textit{9.72} & \textit{17.16} & \textit{5.24} & \textit{0.43} & \textit{15.0} & \textit{17.6} & \textit{10.9} & \textit{11.6} & \textit{13.6} & \textit{5.6} & \textit{29.3} & \textit{18.5} & \textit{26.9} & \textit{33.6} & \textit{48.7} & \textit{48.6} & \textit{10.4} & \textit{14.1} & \textit{94.5} & \textit{55.0} & \textit{15.8} & \textit{7.1} & \textit{31.3} & \textit{24.3} \\
    \midrule
    % AVOD (LiDAR) & 25.87 & 17.01 & 56.98 & 67.97 & 82.9 & 85.0 & 81.8 & 85.5 & 86.4 & 64.8 & 78.8 & 29.6 & 85.4 & 92.0 & 96.1 & 95.8 & 88.3 & 87.9 & 58.9 & 20.0 & 87.8 & 67.7 & 82.9 & 67.9 \\
    AVOD & - & - & - & - & - & - & 81.8 & 85.5 & - & - & - & - & 85.4 & 92.0 & - & - & 88.3 & 87.9 & - & - & - & - & - & - \\
    % AVOD++ (LiDAR) & 5.95 & 12.89 & 15.37 & 10.60 & 21.3 & 23.5 & 12.5 & 10.7 & 38.7 & 27.2 & 51.6 & 59.3 & 36.6 & 46.4 & 72.7 & 72.9 & 13.1 & 18.8 & 121.2 & 100.0 & 39.4 & 36.2 & 48.2 & 46.4 \\
    AVOD++ & - & - & - & - & - & - & 12.5 & 10.7 & - & - & - & - & 36.6 & 46.4 & - & - & 13.1 & 18.8 & - & - & - & - & - & -\\
    % PointPillars (joint) & 3.14 & 5.59 & 31.78 & 34.32 & 34.9 & 39.9 & 22.4 & 22.7 & 56.9 & 71.8 & 62.8 & 96.3 & 58.6 & 63.6 & 89.6 & 92.4 & 44.8 & 36.2 & 120.3 & 100.0 & 68.8 & 77.2 & 65.5 & 70.0 \\
    PointPillars++ & - & - & - & - & - & - & 22.4 & 22.7 & - & - & - & - & 58.6 & 63.6 & - & - & 44.8 & 36.2 & - & - & - & - & - & - \\
    % PointRCNN++ & 8.24 & 18.81 & 9.79 & 3.79 & 18.0 & 22.6 & 7.6 & 5.2 & 22.6 & 6.1 & 49.5 & 51.9 & 48.9 & 60.7 & 76.6 & 75.7 & 17.4 & 35.6 & 120.8 & 100.0 & 34.8 & 29.9 & 47.3 & 45.6 \\
    PointRCNN++ & - & - & - & - & - & - & 7.6 & 5.2 & - & - & - & - & - & - & - & - & - & - & - & - & - & - & - & - \\
    \midrule
    % SECOND++(4) & 4.60 & 8.97 & 11.79 & 8.06 & 16.4 & 17.0 & 8.1 & 4.2 & 22.5 & 7.5 & 48.4 & 66.7 & 34.4 & 41.6 & 73.4 & 73.6 & 10.6 & 14.1 & 121.2 & 100.0 & 35.5 & 26.8 & 44.3 & 41.8 \\
    SECOND++(4) & - & - & - & - & - & - & 8.1 & 4.2 & 22.5 & 7.5 & - & - & 34.4 & 41.6 & - & - & 10.6 & 14.1 & - & - & - & - & - & - \\
    % + Ext. Range & 4.32 & 8.34 & 10.85 & 8.44 & 15.2 & 16.8 & 7.2 & 4.5 & 18.4 & 7.0 & 44.2 & 55.6 & 34.2 & 41.8 & 71.4 & 71.5 & 10.0 & 10.1 & 119.1 & 100.0 & 33.2 & 22.8 & 42.2 & 39.2 \\
    + Ext. Range & - & - & - & - & - & - & \textbf{7.2} & 4.5 & 18.4 & 7.0 & - & - & 34.2 & 41.8 & - & - & 10.0 & \textbf{10.1} & - & - & - & - & - & - \\
    % + BG Removal & 3.57 & 6.02 & 13.88 & 9.61 & 17.5 & 15.6 & 9.9 & 4.0 & 25.9 & 8.9 & 46.1 & 55.6 & 33.2 & 39.1 & 58.4 & 58.3 & 11.9 & 15.4 & 120.8 & 100.0 & 29.6 & 20.5 & 42.0 & 37.7 \\
    + BG Removal & - & - & - & - & - & - & 9.9 & \textbf{4.0} & 25.9 & 8.9 & - & - & 33.2 & 39.1 & - & - & 11.9 & 15.4 & - & - & - & - & - & - \\
    % + Both & 3.67 & 6.08 & 12.67 & 9.41 & 16.3 & 15.5 & 9.3 & 4.0 & 20.6 & 8.0 & 41.9 & 55.6 & 32.9 & 38.7 & 58.4 & 58.3 & 12.0 & 14.1 & 118.2 & 100.0 & 27.3 & 21.3 & 40.1 & 37.5 \\
    + Both & - & - & - & - & - & - & 9.3 & \textbf{4.0} & 20.6 & 8.0 & - & - & 32.9 & \textbf{38.7} & - & - & 12.0 & 14.1 & - & - & - & - & - & - \\
    \midrule
    SECOND++(8) & 4.07 & 7.42 & 11.79 & 9.25 & 15.9 & 16.7 & 8.1 & 5.1 & 23.2 & 8.0 & 43.2 & 51.9 & 33.2 & 39.5 & 70.8 & 70.8 & 9.6 & 13.4 & 119.5 & 100.0 & 28.7 & 18.9 & 42.0 & 38.5 \\
    + Ext. Range & 4.28 & 7.85 & 10.48 & 9.10 & \textbf{14.8} & 17.0 & 7.4 & 4.9 & 17.6 & 8.5 & 34.0 & 48.1 & 34.3 & 40.8 & 72.7 & 72.9 & \textbf{9.5} & 12.1 & 115.3 & 100.0 & 27.8 & 22.8 & 39.8 & 38.8 \\
    + BG Removal & 3.89 & 6.74 & 13.34 & 9.74 & 17.2 & 16.5 & 10.0 & 5.0 & 24.7 & 5.6 & 45.8 & 63.0 & \textbf{32.8} & 40.0 & 63.0 & 62.5 & 11.7 & 14.8 & 116.1 & 95.0 & 26.6 & 21.3 & 41.3 & 38.4 \\
    + Both & 3.84 & 6.66 & 12.28 & 9.66 & 16.1 & \textbf{16.3} & 9.4 & 4.8 & 20.1 & 6.1 & 38.2 & 63.0 & 33.2 & 40.1 & 63.0 & 62.5 & 11.5 & 13.4 & 112.7 & 95.0 & 23.4 & 21.3 & 38.9 & 38.3 \\
    \midrule
    Ours(min) & 15.09 & 25.70 & 5.57 & 0.58 & 20.7 & 26.3 & 15.9 & 17.9 & 15.4 & 11.3 & \textbf{29.3} & \textbf{18.5} & 39.2 & 48.0 & 68.8 & 67.4 & 17.9 & 26.2 & 94.9 & 55.0 & 23.2 & \textbf{15.0} & 38.1 & 32.4 \\
    Ours(avg) & 10.54 & 17.41 & 7.87 & 4.60 & 18.4 & 22.0 & 13.8 & 14.6 & 14.8 & 7.0 & 30.1 & 29.6 & 33.4 & 40.6 & 61.7 & 61.1 & 16.4 & 19.5 & 94.9 & 55.0 & 22.3 & \textbf{15.0} & 35.9 & 30.3 \\
    \midrule
    Ours(avg) w/ \\
    (2.7, 0.9, 0.3)m & 13.41 & 19.65 & 6.16 & 5.13 & 19.6 & 24.8 & 15.8 & 17.1 & 13.9 & 13.6 & 23.6 & 25.9 & 36.2 & 43.3 & 63.6 & 63.2 & 18.3 & 26.2 & 72.0 & 20.0 & 20.2 & 19.7 & 33.0 & 28.6\\
    (2.4, 1.2, 0.6, 0.3)m & 11.26 & 16.65 & 5.94 & 5.69 & 17.2 & 22.3 & 12.9 & 15.0 & 11.5 & 7.5 & 27.2 & 29.6 & 34.7 & 41.2 & \textbf{59.7} & \textbf{59.0} & 15.7 & 19.5 & 80.9 & 25.0 & \textbf{18.4} & 19.7 & 32.6 & 27.1 \\
    (3.2, 1.6, 0.8, 0.4, 0.2)m & 12.36 & 15.10 & 4.67 & 5.36 & 17.0 & 20.5 & 12.8 & 12.7 & \textbf{10.7} & \textbf{4.2} & \textbf{21.5} & 22.2 & 35.8 & 41.2 & \textbf{59.7} & \textbf{59.0} & 17.4 & 16.8 & \textbf{67.8} & \textbf{10.0} & 19.0 & 17.3 & \textbf{30.6} & \textbf{22.9}\\
    \bottomrule
    \end{tabular}
  }
\end{table*}

{
    \bibliographystyle{IEEEtran}
    \bibliography{ref}
}

\end{document}